\documentclass[10pt]{IEEEtran}
\usepackage{amsmath,amsfonts}
\usepackage{amssymb} 
\usepackage{algorithm}
\usepackage{array}
\usepackage{booktabs} 
\usepackage{multirow}
\usepackage{textcomp}
\usepackage{stfloats}
\usepackage{url}
\usepackage{verbatim}
\usepackage{graphicx}
\usepackage{cite}

\usepackage{algpseudocode}
\usepackage{amsthm}
\theoremstyle{definition}
\newtheorem{remark}{Remark}
\theoremstyle{plain}
\newtheorem{theorem}{Theorem}
\newtheorem{lemma}{Lemma}
\newtheorem{prop}{Proposition}

\newtheorem{assu}{Assumption}

\usepackage{subfigure}
\usepackage{float}

\usepackage{mathtools}
\mathtoolsset{showonlyrefs}

\usepackage{xcolor}

\usepackage{amsmath}
\DeclareMathOperator*{\argmax}{arg\,max}

\begin{document}

\title{FedCGD: Collective Gradient Divergence Optimized Scheduling for Wireless Federated Learning}

\author{Tan Chen,~\IEEEmembership{Graduate Student Member,~IEEE}, Jintao Yan,~\IEEEmembership{Student Member,~IEEE}, \\Yuxuan Sun,~\IEEEmembership{Member,~IEEE}, Sheng Zhou,~\IEEEmembership{Senior Member,~IEEE}, Zhisheng Niu,~\IEEEmembership{Fellow,~IEEE}
}
\maketitle

\begin{abstract}
Federated learning (FL) is a promising paradigm for multiple devices to cooperatively train a model. When applied in wireless networks, two issues consistently affect the performance of FL, i.e., data heterogeneity of devices and limited bandwidth. \textcolor{black}{Many papers have investigated device scheduling strategies considering the two issues. However, most of them recognize data heterogeneity as a property of individual devices. In this paper, we prove that the convergence speed of FL is affected by the sum of device-level and sample-level collective gradient divergence (CGD). The device-level CGD refers to the gradient divergence of the scheduled device group, instead of the sum of the individual device divergence. The sample-level CGD is statistically upper bounded by sampling variance, which is inversely proportional to the total number of samples scheduled for local update.} To derive a tractable form of the device-level CGD, we further consider a classification problem and transform it into the weighted earth moving distance (WEMD) between the group distribution and the global distribution. Then we propose FedCGD algorithm to minimize the sum of multi-level CGDs by balancing WEMD and sampling variance, within polynomial time. Simulation shows that the proposed strategy increases classification accuracy on the CIFAR-10 dataset by up to 4.2\% while scheduling 41.8\% fewer devices, and flexibly switches between reducing WEMD and reducing sampling variance.

\end{abstract}


\section{Introduction}
Federated learning (FL), where multiple devices cooperatively train a model by updating the model parameters and sharing the gradients only through a parameter
server, is a popular solution for privacy-preserving distributed
machine learning \cite{mcmahan2017communication,reddi2020adaptive}. When deployed in wireless networks\cite{lim_federated_2020}, FL can leverage the data scattered on mobile devices, and has great potential in enhancing the quality of communication, enabling intelligent and automated vehicles, and the semantic communication system, etc\cite{qin_federated_2021}. 

Implementing FL in wireless networks faces two major issues. Firstly, FL leverages the data samples located on multiple devices, and the distributions of these data samples are usually different due to diverse behaviors and attributes of devices \cite{mcmahan2017communication}. For example, when people surf the Internet via mobile devices, the content preferences are personalized, so the contents that need to be cached are different; when vehicles drive on the road, the power levels of them, their destinations, and the traffic environments together determine their trajectories, so they generally have different trajectories. This phenomenon is called data heterogeneity or non-i.i.d. data. When FL is operated with non-i.i.d. data on devices, since the local objectives diverge from the global objective, it causes the performance degradation\cite{wang2021addressing}. 

Secondly, since the model parameters are shared through wireless channels, the varying channel state and limited bandwidth resources hinder the transmission process. For training tasks with a stringent latency budget\cite{sun_edge_2020} or a tight round deadline\cite{xie2022mob,yan2025dynamic}, transmitting model parameters of all devices may cause unprecedented latency due to some straggler devices, potentially preventing the model from converging within the given time budget\cite{shi_joint_2020}. 

\color{black}
Many studies have jointly considered data heterogeneity and the limited bandwidth resources of FL. However, most of them recognize data heterogeneity as a property of \textit{individual} devices and derive utility values consisting of the data heterogeneity measure and communication measure for each device. Then the joint problem is solved by sorting utility values.

We have a different perspective: Data heterogeneity is a \textit{collective} measure. Even though each device has an imbalanced dataset with individual heterogeneity, when we group their gradients to get a reduced ``collective heterogeneity", the performance may get better.

\color{black}
Therefore, in this work, \textcolor{black}{we focus on searching for such a collective heterogeneity measure and establishing a corresponding relationship between collective data heterogeneity and communication constraints in wireless FL}. The main contributions of our work are summarized as follows.

(1) The convergence speed of wireless FL is analyzed,  revealing its dependence on the sum of \textit{device-level and sample-level collective gradient divergence} (CGD). Device-level CGD refers to the gradient divergence of all the scheduled devices. Sample-level CGD is statistically upper bounded by sampling variance, which is inversely proportional to the total number of samples scheduled for local update.

(2) The classification problem is considered, and device-level CGD is transformed into the weighted earth moving distance (WEMD) between the group data distribution and the global data distribution. Device scheduling strikes a balance between WEMD and sampling variance.

(3) An optimization problem to minimize the sum of multi-level CGDs by device scheduling and bandwidth allocation is formulated. The problem is proven to be NP-hard. A greedy algorithm in polynomial time and a modified coordinate descent algorithm in pseudo-polynomial time are then proposed to solve the problem, which reports average errors of 5.16\% and 0.19\%, respectively.

(4) Simulation results verify that the proposed algorithms increase classification accuracy on the CIFAR-10 dataset by up to 4.2\% while scheduling 41.8\% fewer devices. Besides, the proposed algorithms also show flexibility to switch between reducing WEMD and reducing sampling variance on the CIFAR-10 and CIFAR-100 datasets.

The rest of this paper is organized as follows. Section \ref{Sec-1} introduces the related work. Section \ref{Sec-2} describes the system model and presents the training algorithm. In Section \ref{Sec-3}, convergence analysis of the system is conducted, showing the impact of the device scheduling on the multi-level CGDs. Section \ref{Sec-4} presents the optimization problem and solutions. Section \ref{Sec-5} shows the simulation results and analysis. Finally, Section \ref{Sec-6} concludes the work.
\section{Related Work}
\label{Sec-1}
Many solutions have been proposed to alleviate the impact of data heterogeneity, such as introducing auxiliary data\cite{zhao2018federated,jeong2018communication}, estimating and canceling the degree of data heterogeneity through side information\cite{karimireddy2020scaffold,zhang2022federated}, adopting a personalized FL structure\cite{arivazhagan2019federated}, and so on.
To conquer the problem of limited bandwidth, part of the devices are scheduled for uploading models in order to reduce the round latency. However, scheduling fewer devices may also damage the performance of the aggregated model, so the global model needs more rounds to converge \cite{wang2019adaptive,shi_joint_2020}. A large amount of studies have analyzed the convergence speed or equivalent expressions of wireless FL w.r.t. the scheduling policy\cite{shi_joint_2020,chen_convergence_2021,liu_joint_2022,zhang_communication-efficient_2022,kim_beamforming_2023,yang_asynchronous_2024,chen_energy_2022}, and proposed algorithms to optimize the long-term latency of FL\cite{yang_scheduling_2020,qu_context-aware_2022,albelaihi_green_2022,zhang_joint_2025,luo_joint_2022,pan_contextual_2024}. Despite their values, the relationship between data heterogeneity and scheduling is not investigated.

Some papers have further considered the interplay of data distributions and scheduling.
In \cite{ren_scheduling_2020}, the importance of an edge device's learning update is defined by the gradient variance, and a trade-off between channel quality and update importance is investigated by optimizing the weighted sum of them. The authors of \cite{amiri_convergence_2021} adopt quantization to guarantee that the quantized model can be successfully uploaded in the given time slots. The importance of the model updates is determined by the $l_2$-norm of the quantized gradients, and devices with top-K importance values are scheduled. Ref. \cite{luo_adaptive_2024} also adopts $l_2$-norm of gradients to measure the importance, and proposes a propolistic scheduling scheme considering both the channel quality and the gradient norm. In \cite{liu_data-importance_2021}, data uncertainty, which is declared to be the distance to the decision boundary in SVM and entropy in CNN, is considered as a measure for scheduling. In \cite{xu_client_2021}, a long-term scheduling utility w.r.t the data size and the current training round is considered, and the Lyapunov method is adopted to maximize the overall utility with a long-term energy constraint. The authors of \cite{tang_fedcor_2022} assume that the loss changes of devices in a round follow a GP, and develop an iterative scheduling strategy by predicting the loss change with the upper confidence bound method. Although the impact of data heterogeneity implicitly affects the importance measures of these papers, the correlation between importance measures and data heterogeneity is not well investigated.

In \cite{deng_auction_2022}, a reinforcement learning structure is adopted, where the data sizes and mislabel rates are considered as the measure of data heterogeneity to lead the scheduling action. The authors of \cite{saha_data-centric_2023} model data heterogeneity as the variation of mean and standard deviation among devices, and aim to optimize a utility function of FL composed of data heterogeneity, communication cost, and computation cost. In \cite{lee_data_2023}, the balance score of data is calculated by the Kullback-Leibler divergence between the target distribution and the device distribution. The authors find some device clusters according to their balance scores at the start of training, and use a multi-armed bandit (MAB) to select the cluster with the lowest convergence time every round. In \cite{zhang_stabilizing_2025}, the local-update stability is derived w.r.t. the client-variance upper-bound, and Nesterov accelerated gradient is introduced to improve the stability and accelerate FL. Despite the effectiveness of the data measures in experiments of these works, convergence analysis is not given, and thus, the relationship between the data measures and the convergence speed is unclear.

In \cite{sun_channel_2024}, the convergence speed of FL's loss function with over-the-air aggregation is derived, which shows that each device impacts the performance of FL by the weighted sum of its communication distortion and gradient variance. A probabilistic scheduling policy is then investigated to accelerate the convergence speed. \textcolor{black}{Despite its value, the authors recognize data heterogeneity as a property of individual devices, and add them together while scheduling.}

Recently, some work \cite{zhang_coalitional_2024,zhang_fed-cbs_2023} has started to notice the collective property of data heterogeneity. 
The authors of \cite{zhang_coalitional_2024} prove that the collective earth mover’s distance (EMD) of selected devices is positively correlated to the convergence bound of FL, but do not investigate how to minimize the collective EMD by scheduling.
Ref. \cite{zhang_fed-cbs_2023} expresses the degree of data heterogeneity by quadratic class-imbalance degree (QCID). Theoretical analysis proves that QCID is positively correlated to the convergence speed of FL, and a privacy-preserving method to estimate and update the value of QCID with a combinatorial upper confidence bounds algorithm is proposed. However, the authors assume that a fixed number of devices are scheduled, which is not realistic in wireless FL, where devices face diverse and dynamic channel states. 
\section{System Model}
\label{Sec-2}

\begin{figure}[t]
\centering
\includegraphics[width=0.5\textwidth]{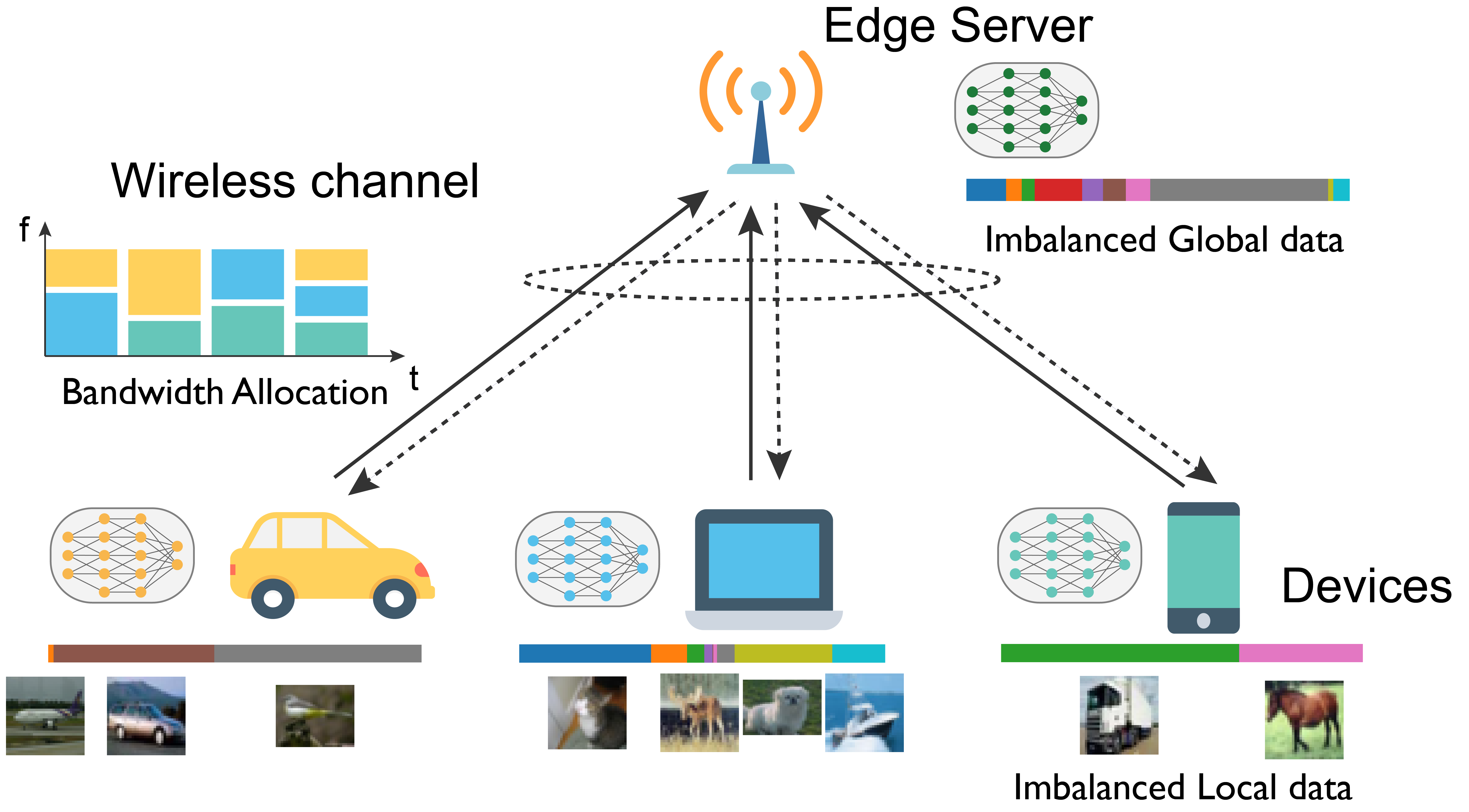}
\caption{Illustration of wireless FL with heterogeneous data.}
\label{FL}
\end{figure}

We consider an FL system in wireless networks, as depicted in Fig. \ref{FL}. A central edge server, which usually represents a base station, covers a limited area around itself. Several mobile devices are connected to the edge server. Due to the dynamic wireless channel state and the computing capability of devices, the set of available devices keeps changing. Each device has its local data, and they cooperatively train a model in a federated manner with the help of the edge server.

The FL task aims to find the connection between inputs $\boldsymbol{x}_i$ and labels $y_i$ in the global dataset $\mathcal{D}=\{\boldsymbol{x}_i,y_i \}$. Each device $v$ holds a local dataset $\mathcal{D}_v$ sampling from the global dataset. Because of the diverse behaviors of devices, the sampling process is not homogeneous. We assume that the sampling follows a distribution $\mathcal{P}$, where $P\{\mathcal{D}_v\sim \mathcal{P}_k\}=p_k, k=1,\cdots,K$ denotes the probability of $\mathcal{D}_v$ to follow the $k$-th pattern of data distribution. Note that $K$ denotes the total number of data patterns. Actually, $K$ can also be infinity, which will not influence the analysis.

For a sample $\{\boldsymbol{x}_i,y_i \}$, let $l_i \left(\boldsymbol{w}\right)$ be the sample loss function for models $\boldsymbol{w}$. The loss function of the $v$-th device is given by $f_v\left(\boldsymbol{w}\right)=\frac1{|\mathcal{D}_v|} \sum_{i\in \mathcal{D}_v}l_i \left(\boldsymbol{w}\right)$, and the global loss function is then determined by an average of the sample loss functions, $F\left(\boldsymbol{w}\right)=\mathbb{E}_v \left[ f_v\left(\boldsymbol{w}\right)\right] $. The training objective of FL is $\min_{\boldsymbol{w}} F\left(\boldsymbol{w}\right)$.

\subsection{Federated Averaging}

We call the stochastic gradient descent (SGD) process on a batch of samples at each device a local update. Denote the models of the $v$-th device, the edge server at the $j$-th edge epoch by $\boldsymbol{w}_{v}^{(j)}, \boldsymbol{w}^{(j)}$, respectively.

The training procedure is illustrated in Fig. \ref{FL}. The edge server initializes with a global model $\boldsymbol{w}^{(0)}$. The training stage consists of the following three steps within each epoch:

    \textbf{1) Edge model distribution}: Let the edge server maintain a covering device set $\mathcal{V}^{(j)}$. 
    The edge server distributes the edge model to all the devices within its coverage, i.e., $\boldsymbol{w}_{v}^{(j)}=\boldsymbol{w}^{(j)}, v\in \mathcal{V}^{(j)}$.

     \textbf{2) Local update}: 
     For each local update, the $v$-th device performs SGD using its local data by 
     \begin{align}
         \boldsymbol{w}_v^{(j)} = \boldsymbol{w}_v^{(j-1)} - \eta\sum\limits_{t=0}^{\tau-1}\nabla f_{v,t}(\boldsymbol{w}_v^{(j, t)}), \label{localupdate}
     \end{align}
     where $\eta$ is learning rate, $\nabla f_{v,t} \left(\boldsymbol{w}_{v}^{(j,t)}\right)\!=\!\nabla f_{v} \left(\boldsymbol{w}_{v}^{(j,t)}, \boldsymbol{\xi}_{v}^{(j,t)}\right)$ is the stochastic gradient of the loss function with data batch $\boldsymbol{\xi}_{v}^{(j,t)}$ sampled from $\mathcal{D}_v$ at $t$-th local update, and $\boldsymbol{w}_v^{(j, t)}=\boldsymbol{w}_v^{(j, t-1)}- \eta\nabla f_t(\boldsymbol{w}_v^{(j, t-1)})$. 
     The local update is repeated $\tau$ times, where $\tau$ is called the local iteration.

    \textbf{3) Device scheduling \& Edge aggregation}: The edge server first updates the covering device set $\mathcal{V}^{(j)}$. Then it chooses the scheduled devices by $\Pi^{(j)} \!=\! \{v|v\in\mathcal{V}^{(j)}, 
    x_{v,j}\!=\!1\}$, where $x_{v,j}\in\{0,1\}$ denotes the scheduling decision. The scheduled devices send models to the edge server for aggregation. The aggregated model is denoted by  
    \begin{align}
        {\boldsymbol{w}}^{(j)}=\sum_{v\in \Pi^{(j)}}\alpha_{v}^{(j)} \boldsymbol{w}_{v}^{(j)},\label{edgeagg}
    \end{align}
    where $\alpha^{(j)}_{v}\!\triangleq\!\frac{|\mathcal{D}_{v}|}{\sum_{v'\in \Pi^{(j)}}|\mathcal{D}_{v'}|}$ denotes the aggregation weight of the $v$-th device.

    The training process consists of $J$ epochs.

\subsection{Communication and Computation Model}
Since we mainly focus on the aggregation step, the computation deadline is set to a fixed value $d_{\text{cp}}$. Devices that have finished the local update before $d_{\text{cp}}$ are candidates for transmitting models. 

As for communication, since the transmit power of the base station is much larger than that of the devices, we only consider the uplink latency. Frequency-Division Multiple Access (FDMA) is adopted, where devices share a total bandwidth of $B$. Assume $B_{v,j}$ is allocated to device $v$ at the $j$-th epoch, and the average channel gain during transmission is denoted by $\Bar{H}_{v, j}$. Then the transmission rate $r_{v,j}$ is calculated by
\begin{align}
    r_{v,j} = B_{v, j}\log_2\left(1+\frac{S\Bar{H}_{v, j}}{B_{v, j}N_0}\right),
\end{align}
where $N_0$ is the spectrum density of the noise. Denote the total data size of the model as $D_{\boldsymbol{w}}$, so the communication latency becomes
\begin{align}
    d_{v, j}=\frac{D_{\boldsymbol{w}}}{r_{v,j}}.
\end{align}

Assume that in each epoch, a latency budget $d_{\text{cm}}$ is set for transmission. If a device does not transmit all its model parameters in time, the model fails to arrive at the edge server.

To investigate the impact of device scheduling and bandwidth allocation on the performance of FL, the convergence analysis of FL is conducted.
\section{Convergence Analysis}
\label{Sec-3}
\textcolor{black}{
In this section, the convergence speed of convex FL is derived. Firstly, the convex loss function is bounded by the rate of centralized learning subtracted by the federal-central (FC) difference. Then the FC difference is split into mainly \textit{two levels of gradient divergence}. The results are also extended to non-convex FL. Finally, the meaning of the two terms and their relationship are elaborated. }

Define $\boldsymbol{v}^{(j)}$ as the virtual centralized model. It monitors the process of centralized learning. It evolves as
    \begin{equation}
\boldsymbol{v}^{(j)}=\boldsymbol{w}^{(j-1)}-\eta\tau\nabla F\left(\boldsymbol{w}^{(j-1)}\right).
    \end{equation}
The FC difference at the $j$-th epoch is denoted by $U_j$, where
    \begin{align}
U_j=\left\lVert\boldsymbol{w}^{(j)}-\boldsymbol{v}^{(j)}\right\rVert.
    \end{align}
$U_j$ represents the difference in the convergence speed between FL and centralized
learning. Also, denote the optimal model parameters by $\boldsymbol{w}^*$.

\begin{assu}\label{assu1} 
We assume the following for all $v$:

\begin{enumerate}
\item $f_v\left(\boldsymbol{w}\right)$ is convex;
\item $f_v\left(\boldsymbol{w}\right)$ is $\rho$-Lipschitz, i.e., 
$\left\lVert {f_v\left(\boldsymbol{w}\right)-f_v\left(\boldsymbol{w}' \right)} \right\rVert\le \rho \left\lVert {\boldsymbol{w}-\boldsymbol{w}'}  \right\rVert$ for any $\boldsymbol{w},\boldsymbol{w}'$;

\item $f_v\left(\boldsymbol{w}\right)$ is $\beta$-smooth, i.e., 
$\left\lVert {\nabla f_v\left(\boldsymbol{w}\right)-\nabla f_v\left(\boldsymbol{w}'\right)} \right\rVert\le \beta \left\lVert {\boldsymbol{w}-\boldsymbol{w}'}  \right\rVert$ for any $\boldsymbol{w},\boldsymbol{w}'$.

\item The expected squared norm of gradients is uniformly bounded, i.e., $\mathbb{E}\left\lVert {\nabla f_v\left(\boldsymbol{w}\right)} \right\rVert^2 \le g^2$ for any $v, \boldsymbol{w}$.

\item Stochastic gradient is independent, unbiased and
variance-bounded, i.e., for any $v, \boldsymbol{w}$, taking the expectation over stochastic data sampling, we have
\begin{align}
    &\mathbb{E}_{\boldsymbol{x}_{v,i}}\left[\nabla f_v(\boldsymbol{w},\boldsymbol{x}_{v,i})\right]=\nabla f_v(\boldsymbol{w}),\\
    &\mathbb{E}_{\boldsymbol{x}_{v,i}}\left[\left\lVert\nabla f_v(\boldsymbol{w},\boldsymbol{x}_{v,i})-\nabla f_v(\boldsymbol{w})\right\rVert^2\right]\le \sigma^2.
\end{align}

\end{enumerate}
\end{assu}

Here assumption 1) to 3) follows \cite{wang2019adaptive,chen_mobility_2025}, assumption 4) follows \cite{li_convergence_2020,feng2022mobility}, and assumption 5) follows \cite{sun_dynamic_2021,reisizadeh_fedpaq_nodate,wang_federated_2022}.

Then we have Proposition \ref{prop1} to describe the overall convergence rate of FL.

\begin{prop}\label{prop1}
    If Assumption \ref{assu1} holds, then after $J$ epochs, for $\eta,\epsilon$ satisfying conditions:
    \begin{enumerate}
        \item[(1)] $\eta\le\frac1{\beta}$,
        \item[(2)] $\omega\eta\tau\left(1\!-\!\frac{\beta\eta\tau}2\right)\!-\!\frac{\rho \mathbb{E}\left[U_j\right]}{\epsilon^2}\!>\!0$ for all $j$, 
        where $\omega\triangleq\min\limits_j\frac{1}{\left\lVert{\boldsymbol{v}^{\left(j\right)}-\boldsymbol{w}^*}\!\right\rVert^2}$,
        \item[(3)]
        $F\left(\boldsymbol{v}^{\left(j\right)}\right)-F\left(\boldsymbol{w}^*\right)\ge\epsilon$ for all $j$,\label{cond}
    \end{enumerate}
 the loss function of FL is bounded by
    \begin{equation}
        F\left(\boldsymbol{w}^{\left(J\right)}\right)-F\left(\boldsymbol{w}^*\right)
        \le\frac1{J\omega\eta\tau\left(1\!-\!\frac{\beta\eta\tau}2\right)\!-\!\frac{\rho}{\epsilon^2}\sum\limits_{j=1}^J \mathbb{E}\left[U_j\right]}.\label{fw}
    \end{equation}
\end{prop}
\begin{proof}
    See Appendix \ref{app1}.
\end{proof}
Proposition \ref{prop1} indicates that the convergence rate of FL depends on two terms, i.e., the convergence rate of centralized learning and the FC difference. Since the convergence rate of centralized learning is not influenced by scheduling, we then focus on bounding the FC difference.

Define a virtual device model $\boldsymbol{\hat{w}}_v^{(j,t)}$ satisfying
\begin{align}
    &\boldsymbol{\hat{w}}_v^{(j,0)} = \boldsymbol{w}_v^{(j-1)},\\
    &\boldsymbol{\hat{w}}_v^{(j,t)} = \boldsymbol{\hat{w}}_v^{(j,t-1)} -\eta\nabla
    f_{v,j,t}(\boldsymbol{w}_v^{(j-1)}),
\end{align}
then we have the following lemma.

\begin{lemma}\label{lemma1}
The FC difference is composed of the following terms
\begin{align}
   \mathbb{E}\left[U_j\right]&\le \eta\tau\underbrace{\left\lVert\sum\limits_{v\in\Pi^{(j)}}\alpha_v^{(j)}\nabla f_v(\boldsymbol{w}^{(j-1)})
    - \nabla F(\boldsymbol{w}^{(j-1)})\right\rVert}_A
    \\&\hspace{5pt}+ \!\eta\tau\mathbb{E}\underbrace{\left\lVert\sum\limits_{v\in\Pi^{(j)}} \alpha_{v}^{(j)} \left(\nabla f_{v,j}(\boldsymbol{w}^{(j-1)})\!-\!\nabla f_{v}(\boldsymbol{w}^{(j-1)})\right)\right\rVert}_B
    \\&\hspace{5pt}+\!\mathbb{E}\left\lVert \boldsymbol{w}_v^{(j)} - \boldsymbol{\hat{w}}_v^{(j,\tau)} \right\rVert.\label{CEcomponents}
\end{align}

\end{lemma}
\begin{proof}
    See Appendix \ref{app2}.
\end{proof}

\begin{figure}[t]
\centering
\includegraphics[width=0.4\textwidth]{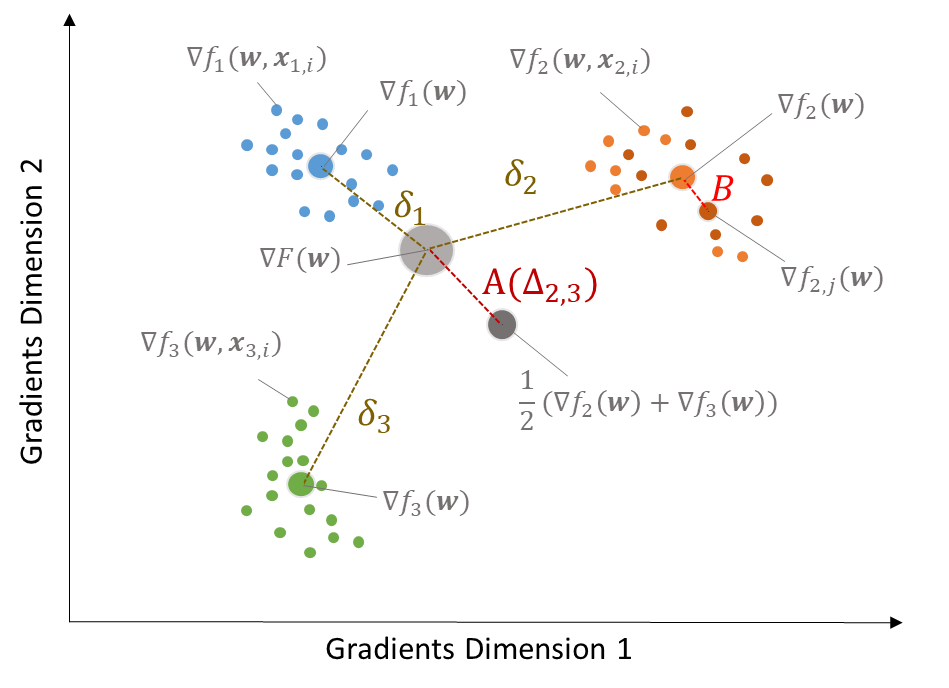}
\caption{Illustration of multi-level collective gradient divergences.}
\label{CFdiff}
\end{figure}

\textcolor{black}{Note that $\nabla f_{v,j}(\boldsymbol{w}^{(j)})$ denotes the gradients of the data samples used by device $v$ during the $j$-th epoch, while $\nabla f_{v}(\boldsymbol{w}^{(j)})$ and $\nabla F(\boldsymbol{w})$ denotes the gradients of the local dataset of device $v$ and the global dataset, respectively. To better understand these gradients, we draw a schematic diagram as shown in Fig. \ref{CFdiff}. We can see high similarity between term A and term B: they are just different levels of gradient divergence, where term A expresses device-level divergence, and term B expresses sample-level divergence. We will elaborate on them along with the last term of Eq. \eqref{CEcomponents} in the following part.}

\subsection{Device-Level Divergence: Device Data Heterogeneity}
During training, we want the samples of each epoch to be representative of the global dataset, so that the model can perform well on all samples from the global dataset. Since the global dataset is the union of all local datasets, when part of the devices are scheduled, the samples of the epoch are actually drawn from part of the global dataset, and thus may not be representative enough of the global dataset. This is how data heterogeneity comes.

In FL, the dataset is locally kept on devices, so the representativeness cannot be directly achieved. However, some substitutes can be found. Here we denote the data heterogeneity $\Delta^{(j)}$ by
\begin{align}
    \Delta^{(j)} = \left\lVert\sum\limits_{v\in\Pi^{(j)}}\alpha_v^{(j)}\nabla f_v(\boldsymbol{w}^{(j-1)})
    - \nabla F(\boldsymbol{w}^{(j-1)})\right\rVert,\label{dataerr}
\end{align}
which is exactly device-level gradient divergence.

\begin{remark}
    From the definition, we see that the divergence we consider is actually ``collective divergence", since it expresses the divergence between average gradients of the scheduled device group and the global gradients. Specifically, denote the individual device divergence by $\delta^{(j)}_v = \left\lVert\nabla f_v(\boldsymbol{w}^{(j-1)})
    - \nabla F(\boldsymbol{w}^{(j-1)})\right\rVert$, 
     and scheduling those devices with small $\delta^{(j)}_v$ does not lead to a minimal $\Delta^{(j)}$. Instead, we need to choose ``complementary" devices such that their ``collective divergence" is small. For example, in Fig. \ref{CFdiff}, although $\delta_1$ is much smaller than $\delta_2,\delta_3$, when we choose device 2 and 3 to compose $\Delta_{2,3}$, it is smaller than $\delta_1$.
\end{remark}

\begin{remark}
The ``data heterogeneity" that we consider is the ``gradient heterogeneity". Gradients imply the information of data, since they are calculated by a transformation, i.e., forward and backward of the data.
\textcolor{black}{It indicates that probably not all differences in data samples result in degraded performance. If two devices have two datasets with totally different samples, but the gradients generated by them have a subtle difference, then the heterogeneity may have little influence on the performance of FL.}
\end{remark}

\subsection{Sample-Level Divergence: Batch Sampling Variance}
In the training process, a mini-batch is sampled from the local dataset for a single epoch of training, because the local dataset is usually too large to go through at every epoch. Since each sample is unique, ignoring some samples introduces error, which we call \textbf{sampling variance}. The sampling variance is derived by bounding sample-level gradient divergence:
\begin{lemma}\label{lemma2}
Sample-level gradient divergence is bounded by
\begin{align}
\mathbb{E}\left\lVert\sum\limits_{v\in\Pi^{(j)}} \!\alpha_{v}^{(j)} \left(\nabla f_{v,j}(\boldsymbol{w}^{(j-1)})\!-\!\nabla f_{v}(\boldsymbol{w}^{(j-1)})\right)\right\rVert
\le\frac{\sigma}{\sqrt{|\Pi^{(j)}|b}},\label{samperr}
\end{align}
where $b$ denotes the batch size.
\end{lemma}
\begin{proof}
    Taking assumption 5) in Assumption \ref{assu1} and considering the independence among samples and devices, we prove this lemma. The detailed proof is provided in Appendix \ref{app3}.
\end{proof}
Lemma \ref{lemma3} shows that sampling variance is inversely proportional to the total number of samples of the scheduled devices. The result can be easily generalized to devices with varying batch sizes by replacing $|\Pi^{(j)}|b$ with $\sum\limits_{v\in|\Pi^{(j)}|}b_v$.

\color{blue}\color{black}
\begin{remark}
   Despite the similarity of device-level and sample-level gradient divergence, in Sections \ref{Sec-3}-A and \ref{Sec-3}-B, they are treated in different ways. This is a different granularity of treating the divergence. If we consider manually selecting items, the divergence of each item is first calculated, and then ``collective divergence" is examined; if we adopt random selection, then the divergence of a single item is not needed, and the divergences are statistically bounded by a variance. Since this paper focuses on scheduling devices, we manually select devices while randomly selecting samples.
\end{remark}
\color{black}

\subsection{Local Iteration Bias}
In centralized learning or traditional distributed learning ($\tau=1$), the model is always synchronized for every device. However, in FL, the device models are shared every $\tau$ local iterations ($\tau>1$), so they will have subtle differences during local update. This brings bias to the model, which is expressed by the difference between the virtual device model and the real device model. We derive the bias by the following lemma
\begin{lemma}\label{lemma3}
The local iteration bias of FL satisfies
\begin{align}
    \mathbb{E}\left\lVert \boldsymbol{w}_v^{(j)} - \boldsymbol{\hat{w}}_v^{(j,\tau)} \right\rVert \le \frac{1}{2} \tau(\tau-1)\eta \beta g.\label{modelerr}
\end{align}

\end{lemma}
\begin{proof}
    See Appendix \ref{app4}.
\end{proof}

\subsection{Multi-Level Collective Gradient Divergences}

\begin{theorem}\label{theorem1}
The FC difference of the $j$-th epoch satisfies 
\begin{align}
\mathbb{E}\left[U_j\right] & 
\leq
\frac{1}{2} \tau(\tau-1)\eta \beta g
+\eta\tau  \left(\frac{\sigma}{\sqrt{|\Pi^{(j)}|b}} + \Delta^{(j)}\right).
\end{align}

\end{theorem}
\begin{proof}
    Taking Eq. \eqref{dataerr}-\eqref{modelerr} into Eq. \eqref{CEcomponents}, we prove the theorem.
\end{proof}
Theorem \ref{theorem1} indicates that when $\tau$ is small, device-level gradient divergence, sampling variance, and local iteration bias together impact the performance of FL. As $\tau$ increases, local iteration bias becomes dominant since it goes up at a quadratic scale. However, when $\tau$ is too large, the FC difference may be larger than the convergence rate of centralized learning as stated in \eqref{fw}. Therefore, $\tau$ should be small, or the training resources may be wasted, and FL may also fail to converge.

\textcolor{black}{Besides, when $\tau$ and $\eta$ are fixed, device scheduling impacts the FC difference mainly by device-level gradient divergence and sampling variance. Since they both originate from \textbf{collective gradient divergence} (CGD), but at different levels, we call them \textbf{multi-level CGDs}. Theorem \ref{theorem1} can be extended to the non-convex loss.} 

\begin{theorem}\label{theo2}
    (Non-convex). If 2)-4) of Assumption \ref{assu1} holds and $\eta\le\frac1{2\beta\tau}$, then the average squared gradients of $F\left(\boldsymbol{w}^{(j)}\right)$ is bounded by
    \begin{align}
    \frac1J\sum_{j=0}^{J-1}\left\lVert\nabla F\left(\boldsymbol{w}^{(j)}\right)\right\rVert^2
    \le &
\frac{\mathbb{E}\left[F\left(\boldsymbol{w}^{(0)}\right)\right] - \mathbb{E}\left[F\left(\boldsymbol{w}^{(J)}\right)\right]}{JK(\eta,\tau)} \\
    &+ \left(\frac1{2\eta\tau}+\beta\right)\frac1{JK(\eta,\tau)}\sum_{j=1}^{J}\mathbb{E}\left[U_j^2\right],
\end{align}
where $K(\eta,\tau)=\left(\frac{\eta\tau}2-\beta\eta^2\tau^2\right)$. Furthermore, $\mathbb{E}\left[U_j^2\right]$ is bounded by
\begin{align}
\mathbb{E}\left[U_j^2\right]
\le&
2\mathbb{E}\left\lVert\sum\limits_{v\in\Pi^{(j+1)}} \alpha_{v}^{(j+1)} \left(\boldsymbol{w}_v^{(j+1)}
-\boldsymbol{\hat{w}}_v^{(j,\tau)}\right)\right\rVert^2
\\&+2\left(\frac{\sigma}{\sqrt{|\Pi^{(j)}|b}} + \Delta^{(j)}\right)^2.
\end{align}
\end{theorem}
\begin{proof}
    See Appendix \ref{app5}.
\end{proof}

Theorem \ref{theo2} proves that the non-convex loss of FL is affected by the square of multi-level CGDs. Therefore, reducing multi-level CGDs leads to better performance for both convex and non-convex FL.

Intuitively, there is a balance between the two CGDs. If we want to minimize sampling variance, we have to schedule as many devices as possible, but that may lead to a high device-level CGD; meanwhile, if we focus on decreasing device-level CGD, some devices may be excluded, thus increasing sampling variance. Therefore, these two CGDs should be simultaneously considered to optimize the performance of FL.

\section{Optimization Problem}
\label{Sec-4}
Based on the convergence bound derived in Section \ref{Sec-3}, we consider optimizing the performance of FL by minimizing the sum of multi-level CGDs. We focus on the device scheduling step. The FDMA system is considered with a communication deadline $d_{\text{cm}}$ and total bandwidth $B$ as stated in Section \ref{Sec-2}. For the $v$-th device, the scheduling variable $x_v$ and the corresponding bandwidth allocation variable $B_v$ are decided. The problem is formulated as
\begin{subequations}
\begin{align}
&\min_{\{x_{v},B_{v}\}} \frac{\sigma}{\sqrt{\sum\limits_{v\in\mathcal{V}}x_{v}b}} + \Delta\\
&s.t. \quad x_{v} \in\{0,1\}, v\in\mathcal{V},\label{p0a}\\
&\quad\quad \frac{x_{v}D_{\boldsymbol{w}}}{B_{v}\log_2\left(1+\frac{S\Bar{H}_{v, j}}{B_{v}N_0}\right)}\le 
d_{\text{cm}},v\in\mathcal{V},\label{p0b}\\
&\quad\quad \sum_{v\in\mathcal{V}} x_{v}B_{v}\le B,\label{p0c}\\
&\quad\quad B_v\ge 0, v\in\mathcal{V}.\label{p0d}
\end{align}
\end{subequations}
Here, the superscript $j$ is removed, since there are no correlations across rounds.

The key challenge of the problem is that $\Delta$ is not available. According to the definition, $\Delta$ is composed of the device gradient and the global gradient. For the device gradient, the edge server gets access to it only when the device uploads its model, but if all devices upload models, the resource is already utilized, so scheduling is not necessary. For the global gradient, it is only accessible when the edge server holds the global dataset, which is impossible in FL. 

To combat the challenge, we consider a specific scenario of the learning task.

\subsection{Objective Transformation for Classification Problem}
The loss function of the classification problem satisfies the following properties:
\begin{align}
    F\left(\boldsymbol{w}\right)&=\sum\limits_{c=1}^{C}p_c\mathbb{E}_{\boldsymbol{x}|y\!=\!c}[\log g_i\left(\boldsymbol{w}\right)],
\end{align}
where $C$ is the number of classes, and $p_c$ denotes the portion of the samples belonging to class $c$ over the dataset. Denote the device distribution by $\boldsymbol{p}_{v}^{[j]}=\left[p_{v,1}^{[j]},\cdots,p_{v,c}^{[j]},\cdots,p_{v,C}^{[j]}\right]$ and the global distribution by $\boldsymbol{p}=[p_{1},\cdots,p_{c},\cdots,p_{C}]$, separately.

Assume $|\mathcal{D}_{v}|$ are equal for any $v$, so that $\alpha_v=\frac{1}{\sum\limits_{v \in \mathcal{V}} x_{v}}$. Then we can bound the data heterogeneity by 
\begin{align}
    \Delta\le
    &\left\lVert \sum_{c=1}^C \left(\frac{\sum\limits_{v \in \mathcal{V}} x_{v} p_{v, c}}{\sum\limits_{v \in \mathcal{V}} x_{v}}-p_c\right)\nabla_{\boldsymbol{w}}\mathbb{E}_{\boldsymbol{x}|y=c}[\log g_i(\boldsymbol{w})] \right\rVert\\
\le&  
\sum_{c=1}^C \left|\frac{\sum\limits_{v \in \mathcal{V}} x_{v} p_{v, c}}{\sum\limits_{v \in \mathcal{V}} x_{v}}-p_c\right|
 G_c\label{G1}, \\
\end{align}
where $G_c=\left\lVert\nabla_{\boldsymbol{w}}\mathbb{E}_{\boldsymbol{x}|y=c}[\log g_i\left(\boldsymbol{w}\right)]\right\rVert
$ denotes the expectation of class gradients. Eq. \eqref{G1} indicates that for a classification problem, device-level CGD is bounded by the weighted earth moving distance (WEMD) between the collective data distribution of the scheduled devices and global data distribution. So the optimization problem becomes
\begin{align}
\textbf{P0}: &\min_{\{x_{v},B_{v}\}} \frac{\sigma}{\sqrt{\sum\limits_{v\in\mathcal{V}}x_{v}b}} + \sum_{c=1}^C\left| 
\frac{\sum\limits_{v \in \mathcal{V}} x_{v} p_{v, c}}{\sum\limits_{v \in \mathcal{V}} x_{v}}-p_c \right| G_c\\
&s.t. \quad \text{constraints \eqref{p0a},\eqref{p0b},\eqref{p0c},\eqref{p0d}}.
\end{align}

The minimum bandwidth for each scheduled device in \textbf{P0} is calculated by \cite{shi_joint_2020}
\begin{align}
    B^*_{v}=-\frac{D_{\boldsymbol{w}}\ln 2}{d_{\text{cm}}\left(W(-\Gamma_{v} e^{-\Gamma_{v}})+\Gamma_{v}\right)},\label{lambert}
\end{align}
where $W(\cdot)$ is the Lambert-W function, and $\Gamma_{v}=\frac{N_0 D_{\boldsymbol{w}}\ln2}{d_{\text{cm}}S\Bar{H}_{v}}$. Any feasible bandwidth $B_{v}$ with $x_v=1$ should satisfy $B_{v}>B^*_{v}$ because the LHS of constraint \eqref{p0b} monotonically decreases as $B_{v}$ increases.

Therefore, we can let constraint \eqref{p0b} be strict when $x_v=1$ without impacting the optimal value. 
Then we reformulate the problem as
\begin{align}
&\textbf{P1}: \min_{\{x_{v}\}} \frac{\sigma}{\sqrt{\sum\limits_{v\in\mathcal{V}}x_{v}b}} + \sum_{c=1}^C\left| 
\frac{\sum\limits_{v \in \mathcal{V}} x_{v} p_{v, c}}{\sum\limits_{v \in \mathcal{V}} x_{v}}-p_c \right| G_c\\
&s.t. \quad \text{constraint \eqref{p0a}},\\
&\quad\quad \sum_{v\in\mathcal{V}} x_{v}B^*_{v}\le B,\\
&\quad\quad x_v B_v^*\ge 0, v\in\mathcal{V}.\\
\end{align}
When the channel condition is very bad, even if the bandwidth goes to infinity, the model can not be uploaded in time, and Eq. \eqref{lambert} produces a minus value. Therefore, the last constraint is necessary to prevent $x_v$ from taking 1 when $B_v^*$ is minus.

\textbf{P1} is different from former device scheduling problems because of the collective property. For example, when $\boldsymbol{p}_1=\boldsymbol{p}_2=[0.51,0.49],\boldsymbol{p}_3=[0.8,0.2],\boldsymbol{p}_4=[0.2,0.8]$ and $\boldsymbol{p}=[0.5,0.5]$, we tend to choose device 3 and 4 instead of device 1 and 2. Even though device 3 and 4 are “bad” individual devices, when they are grouped, they form a better device group. 

\subsection{Algorithms for \textbf{P1}}
We first prove the NP-hardness of \textbf{P1}.

\begin{lemma}\label{lemma4}
The partition problem is reducible to \textbf{P1} in polynomial time.
\end{lemma}
\begin{proof}
    The partition problem searches for a way to divide a set of integers $r_1,\cdots,r_S$ into two subsets with the same sum, where $S$ denotes the total number of candidate integers. 
    
    Denote the sum of these integers by $r_{\text{sum}}$, and the problem can be divided into subproblems to verify the existence of a subset $\mathcal{R}$, such that
    \begin{align}        \sum\limits_{r_i\in \mathcal{R}}r_i=\frac{r_{\text{sum}}}2, |\mathcal{R}|=s,
    \end{align}
    where $s$ takes values from 1 to $\lceil \frac{S}2\rceil$.

    Meanwhile, each sub-problem is corresponding to \textbf{P1} if we set $C=1,|\mathcal{V}|=S,  p_{v,0}=r_v, p_0=\frac{c_{\text{sum}}}{2s},B_v=\frac Bs$, and proper $\sigma,b,G_c$ such that $\frac{\sigma}{\sqrt{b}G_c}(\frac1{\sqrt{s-1}}-\frac1{\sqrt{s}})\ge \frac{c_{\text{sum}}}{2s}$, which is proved in Appendix \ref{app6}. Therefore, the partition problem is reducible to problem \textbf{P1} in $O(|\mathcal{V}|)$ time.
\end{proof}

The partition problem has been proved to be NP-complete, and \textbf{P1} is not NP since an optimal solution can not be verified in polynomial time, so \textbf{P1} is NP-hard.

Next, we propose two algorithms to solve \textbf{P1}.

\subsubsection{Greedy Scheduling}
We first propose a heuristic algorithm. Denote the scheduled set by $\Pi$ and initialize with $\Pi=\emptyset$. Let $W(\Pi)=\sum_{c=1}^C\left| 
\frac{\sum\limits_{v \in \Pi} x_{v} p_{v, c}}{\sum\limits_{v \in \Pi} x_{v}}-p_c \right|G_c$ be the WEMD of $\Pi$. 
We iteratively choose $v_k$ by

\begin{align}
    v_k = \argmax_v W(\Pi)-W(\Pi\cup v).
\end{align}

If $v_k$ satisfies $W(\Pi)-W(\Pi\cup v_k)+\frac{\sigma}{\sqrt{b}}(\frac1{\sqrt{|\Pi|}}-\frac1{\sqrt{|\Pi+1|}})\ge0$, meaning that adding $v_k$ reduces the objective, we add it to $\Pi$; else, we stop the iteration. Since the value of $W(\Pi\cup v)$ needs to be recalculated at each iteration, the complexity of the greedy algorithm is $O(|\mathcal{V}|^2).$

\begin{algorithm}[t]
\caption{Greedy Scheduling (GS) Algorithm}
\label{alg1}
\begin{algorithmic}[1] 
\Require $\sigma,G_c$
    \State Initialize $\Pi=\Phi, \mathcal{S}=\mathcal{V}$.
    \While{$|\mathcal{S}|>0$}
        \For{$v \in \mathcal{S}$}
            \State Calculate $W(\Pi\cup v_k)$. 
        \EndFor
        \State Greedily schedule $v_k = \argmax_v W(\Pi)-W(\Pi\cup v_k)$.
        \If{$W(\Pi)-W(\Pi\cup v_k)+\frac{\sigma}{\sqrt{b}}(\frac1{\sqrt{|\Pi|}}-\frac1{\sqrt{|\Pi+1|}})\ge0$}
            \State Update $\Pi\gets\Pi\cup v_k, \mathcal{S}\gets\mathcal{S}\setminus v_k$.
        \Else
            \State Break.
        \EndIf
    \EndWhile
    \State\Return $\Pi$
\end{algorithmic}
\end{algorithm}

\subsubsection{Coordinate Descent}
We also use a fix-sum coordinate descent (FSCD) algorithm to get a proximately optimal solution. The coordinate descent algorithm \cite{nan_robust_2025}
first randomly initialize $\boldsymbol{x}_0$ with each element taking values of 0 or 1, then repeat to enumerate all 1-step neighbors of $\boldsymbol{x}_k$ as 
\begin{align}
    \boldsymbol{x}_{k,v} = [x_{k,1},\cdots,1-x_{k,v},\cdots,x_{k,|\mathcal{V}|}].
\end{align}
Then the one with the minimum objective value is chosen by
\begin{align}
    \boldsymbol{x}_{k+1}=\argmax
    _{v}W(\boldsymbol{x}_{k,v}).
\end{align}

The iteration is stopped when no 1-step neighbor leads to a decrease of the objective, i.e., $\boldsymbol{x}_{k+1}=\boldsymbol{x}_{k}$. The time complexity of the algorithm is $O(t|\mathcal{V}|)$, where $t$ is the maximum iteration step.

In our problem, if the current iteration point and the global optimal point have the same schedule number, coordinate descent takes at least two steps to transit to the global optimal point: reverse a ``1" to ``0", and another ``0" to ``1", or reverse a ``0" to ``1" first, then another ``1" to ``0".
However, this introduces extra risk: if a ``1" is reversed to ``0", sampling variance increases, so less WEMD is needed to make the transit happens; if a ``0" is reversed to ``1", the total bandwidth consumption increases, and thus the bandwidth constraint may be violated. 

\begin{algorithm}[t]
\caption{Fix-Sum Coordinate Descent (FSCD) Algorithm}
\label{alg2}
\begin{algorithmic}[1] 
\Require $\sigma,G_c$
    \For{$S = V:-1:0$}
        \State Initialize schedule array $\boldsymbol{x}^{(S)}_{0}$ where $S$ devices with least bandwidth requirements are set to 1.
        \While{True}
            \State $\boldsymbol{x}^{(S)}_{k+1}=\argmax
    _{v}\boldsymbol{x}^{(S)}_{k,v,u}$.
            \If{$\boldsymbol{x}^{(S)}_{k+1}=\boldsymbol{x}^{(S)}_{k}$}
                \State $\boldsymbol{x}^{(S)}\gets\boldsymbol{x}^{(S)}_{k}$.
                \State Break.
            \EndIf
        \EndWhile
        \If{$W(\Pi) + \frac{\sigma}{\sqrt{Sb}} \le \frac{\sigma}{\sqrt{(S-1)b}}$}
            \State Break.
        \EndIf\Comment{Early Exit}
    \EndFor
    \State\Return $\boldsymbol{x}_{\text{FSCD}}=\argmax_{\boldsymbol{x}^{(S)}}\left[f(\boldsymbol{x}^{(S)})+ \frac{\sigma}{\sqrt{Sb}}\right]$
\end{algorithmic}
\end{algorithm}

\begin{algorithm}[t]
\caption{FedCGD Scheduling Algorithm}
\label{alg3}
\begin{algorithmic}[1] 
\State Initialize $\boldsymbol{w}^{(0)}$
    \For{$j=1,2,\cdots,J$}
        \State The edge server broadcasts model $\boldsymbol{w}^{(j)}$ to all available devices $\mathcal{V}^{(j)}$.
        \For{device $v \in \mathcal{V}^{(j)}$}
            \State Receive $\boldsymbol{w}^{(j)}$ and update according to \eqref{localupdate}.
            \State Estimate $\hat{\sigma}_v^{(j)}$ according to \eqref{local_est_sigma}.
            \State Calculate $\boldsymbol{p}_{v}^{[j]}$ over the sampled data.
            \State Send $\hat{\sigma}_v^{(j)},\boldsymbol{p}_{v}^{[j]}$ to the edge server.
        \EndFor
        \State Receive $\hat{\sigma}_v^{(j)},\boldsymbol{p}_{v}^{[j]}$ from all devices.
        \State Calculate $\hat{\sigma}^{(j)}$ according to \eqref{global_est_sigma}.
        \State Call Alg. \ref{alg1} or Alg. \ref{alg2} with $\hat{\sigma}^{(j)},\hat{G}^{(j-1)}$ to get the scheduling policy $\Pi^{(j)}$.
        \For{device $v \in \Pi^{(j)}$}
            \State Send $\boldsymbol{g}_v^{(j)}=\boldsymbol{w}_v^{(j+1)}-\boldsymbol{w}_v^{(j)}$ to the edge server.
        \EndFor
        \State Calculate $\boldsymbol{w}^{(j+1)}$ by \eqref{edgeagg}.
        \State Estimate $\hat{G}^{(j)}$ by \eqref{G_est}.
    \EndFor
\end{algorithmic}
\end{algorithm}

Therefore, to make it easier for the points with the same schedule number to transit, we fix the schedule number to $S$. To initialize, devices with $S$ devices with the least bandwidth requirements are set to ``1", while others are set to ``0". In each iteration, instead of considering 1-step neighbors, we consider all 1-step transit, which is denoted by
\begin{align}
    \boldsymbol{x}^{(S)}_{k,v,u} = [x^{(S)}_{k,1},\cdots,1-x^{(S)}_{k,v},\cdots,1-x^{(S)}_{k,u},\cdots,x^{(S)}_{k,|\mathcal{V}|}],&\\
     v\in\Pi, u\in \mathcal{V}\setminus\Pi.&
\end{align}
Then we update $\boldsymbol{x}^{(S)}_{k}$ by $\boldsymbol{x}^{(S)}_{k+1}=\argmax
    _{v}\boldsymbol{x}^{(S)}_{k,v,u}$ until no 1-step transit makes the objective decrease. We obtain the result by enumerating $S$ from $|\mathcal{V}|$ to 1 and choosing the best decision by $\boldsymbol{x}_{\text{FSCD}}=\argmax_S\boldsymbol{x}^{(S)}$.

During the iteration of FSCD, an early exit scheme can be adopted, i.e., if the local optimal point $\boldsymbol{x}^{(S)}$ satisfies
\begin{align}
    W(\Pi) + \frac{\sigma}{\sqrt{Sb}} \le \frac{\sigma}{\sqrt{(S-1)b}},
\end{align}
then the iteration stops, because the schedule decisions with scheduled devices fewer than $S$ can not get a smaller objective. Since the space of 1-step transit is $O(|\mathcal{V}|^2)$, the complexity of the FSCD algorithm is $O(t|\mathcal{V}|^2)$, where $t$ is the total iteration steps.

\subsection{Parameter Estimation}
The precise values of $\sigma$ and $G_c$ in \textbf{P1} can not be obtained in simulation, because the loss functions over the whole device dataset and over the global dataset are not calculated during one round. Therefore, we have to estimate them.

Firstly, $\sigma$ can be estimated by the variance of gradients over the first batch
\begin{align}
    &\hat{\sigma}_v^{(j)} = \sqrt{\frac1b \sum\limits_{i=1}^b\left[\left\lVert\nabla l_{v,j,0,i}(\boldsymbol{w}^{(j)})\!-\!\nabla f_{v,j,0}(\boldsymbol{w}^{(j)})\right\rVert^2\right]},\label{local_est_sigma}\\
    &\hat{\sigma}^{(j)} = \sqrt{\sum\limits_{v \in \mathcal{V}^{(j)}}\alpha_{v}^{(j)}(\hat{\sigma}_v^{(j)})^2}.\label{global_est_sigma}
\end{align}

Secondly, inspired by Ref. \cite{shi_joint_2020}, the gradients of devices and the global gradient can be estimated by multiple local updates of mini-batch SGD as follows
\begin{align}
    & \nabla \hat{f}_v\left(\boldsymbol{w}^{(j)} \right)=
    \frac{\boldsymbol{w}_v^{(j+1)}-\boldsymbol{w}_v^{(j)}}{\tau \eta}, \\
    & \nabla\hat{F}\left(\boldsymbol{w}^{(j)} \right)=
    \sum_{v \in \mathcal{V}^{(j)}}\alpha_{v}^{(j)}
    \nabla f_v\left(\boldsymbol{w}_v^{(j)} \right).
\end{align}

If each device has one class of data, denote the scheduled devices with class $c$ as $\Pi_c^{(j)}$, and $G_c^{(j)}$ is estimated by
\begin{align}
    \hat{G}_c^{(j)} = \max\limits_{v\in\Pi_c}\frac{\left\lVert\hat{f}_v\left(\boldsymbol{w}^{(j)} \right)-\hat{F}\left(\boldsymbol{w}^{(j)}\right)\right\rVert}{\lVert\boldsymbol{p}_{v}^{[j]}-\boldsymbol{p}\rVert_1}.\label{Gc_est}
\end{align}
If devices have more than one class of data, $G_c^{(j)}$ is hard to derive. We denote $G^{(j)}=\max_c G_c^{(j)}$, and estimated $G^{(j)}$ by

\begin{align}
    \hat{G}^{(j)} = \max\limits_v\frac{\left\lVert\hat{f}_v\left(\boldsymbol{w}^{(j)} \right)-\hat{F}\left(\boldsymbol{w}^{(j)}\right)\right\rVert}{\lVert\boldsymbol{p}_{v}^{[j]}-\boldsymbol{p}\rVert_1}.\label{G_est}
\end{align}

The overall procedure of this Section is summarized in Algorithm \ref{alg3}, \textcolor{black}{named Federated Averaging with minimizing CGD (FedCGD)}.
\section{Simulation Results}
\label{Sec-5}

In this section, simulations on the optimization problem are conducted first to show the effectiveness of the proposed algorithms, then FL training simulations are adopted to verify the performance of the proposed method. 

\subsection{Simulation Settings} 
We consider a wireless FL system on the cellular network, where $V=64$ devices are located around an edge server in a cell of radius $250$ m. Since devices have limited computing capacity and may have to finish other tasks, we assume that each of them has a probability $p_{\text{a}}$ to be available for FL and finish the local update before the deadline. We set the available probability to $p_{\text{a}}=0.3$ for each device in our simulation. Only the uplink channel is considered since we focus on the aggregation process. We adopt the channel for 6 GHz in the UMi-Street Canyon scenario proposed by TR 38.901 specification of 3GPP \cite{3gpp_2024} with total bandwidth $B=20$ MHz. The path loss of LOS and NLOS channels are denoted by
\begin{align}
    PL_{\text{LOS}}=32.4+21\log_{10}(d_{\text{3d}})+20\log_{10}(f),\\
    PL_{\text{NLOS}}=32.4+31.9\log_{10}(d_{\text{3d}})+20\log_{10}(f),
\end{align}
where $d_{\text{3d}}$ is the 3d distance between the device and the edge server, and $f$ is the carrier frequency. The standard deviation of shadow fading of LOS and NLOS channels is $4$ dB and $8.2$ dB, respectively. Fast fading is not considered since we focus on the average communication rate during the upload deadline $d$. The LOS probability is calculated by 
\begin{align}
    Pr_{\text{LOS}}=\frac{18}{d_{\text{2d}}}+\exp\left[\left(-\frac{d_{\text{2d}}}{36}\right)\left(1-\frac{18}{d_{\text{2d}}}\right)\right],
\end{align}
where $d_{\text{2d}}$ is the 2d distance between the device and the edge server.

The transmit power of devices is set to $23$ dBm, and the noise figure of the receiver at the edge side is set to $6$ dB. All channel parameters are listed in Table \ref{table1}. 

\begin{table}[t]
\centering
\caption{Summary of Channel Parameters}
  \label{table1}
  \renewcommand\arraystretch{1.1}
\begin{tabular}{|c|c|} 
\hline 
\textbf{Parameters} & \textbf{Values} \\ \hline 
Channel Bandwidth & 20MHz \\ \hline
Transmission Power & 23dBm \\ \hline
Carrier Frequency  & 3.5GHz \\ \hline
Device/Edge Server Antenna height & 1.5m/10m \\ \hline
LOS/NLOS Shadow Fading Std. Dev. & 4dB/8.2dB \\ \hline
Noise Power Spectral Density & -174dbm/Hz \\ \hline
Receiver Noise Figure & 6dB \\ \hline
Upload Deadline & 2s/120s \\ \hline
\end{tabular}
\end{table}

FL simulations are conducted on the CIFAR-10 dataset and CIFAR-100 dataset\footnote{https://www.cs.toronto.edu/~kriz/cifar.html} for the image classification task. Both of them have 50000 training samples and 10000 testing samples in total. CIFAR-10 has 10 classes, with 5000 training samples and 1000 test samples for each class; CIFAR-100 has 100 classes, with 500 training samples and 100 test samples for each class. Two types of data division strategies are considered:

(1) Sort and partition: the dataset is sorted by label, partitioned into shards, and allocated to devices. Each device gets $l$ shards, and smaller $l$ statistically indicates larger data heterogeneity. Besides, as the devices participating in FL training may not be representative enough, we assume that the total dataset of the participating devices may be biased \cite{zhang_fed-cbs_2023}. Let the total dataset contain $n_1$ samples of the first half classes and $n_2$ samples of the second half classes. Denote the imbalance ratio as $r=\frac{n2}{n1}$. When $r=1$, the total dataset is balanced; when $r\neq1$, the total dataset is imbalanced. In the simulation, we consider the cases of $r=1,3,9$.

\begin{figure}[!t]
	\centering
	\subfigure[Average iterations.]{\label{fig18}			
		\includegraphics[width=0.22\textwidth]{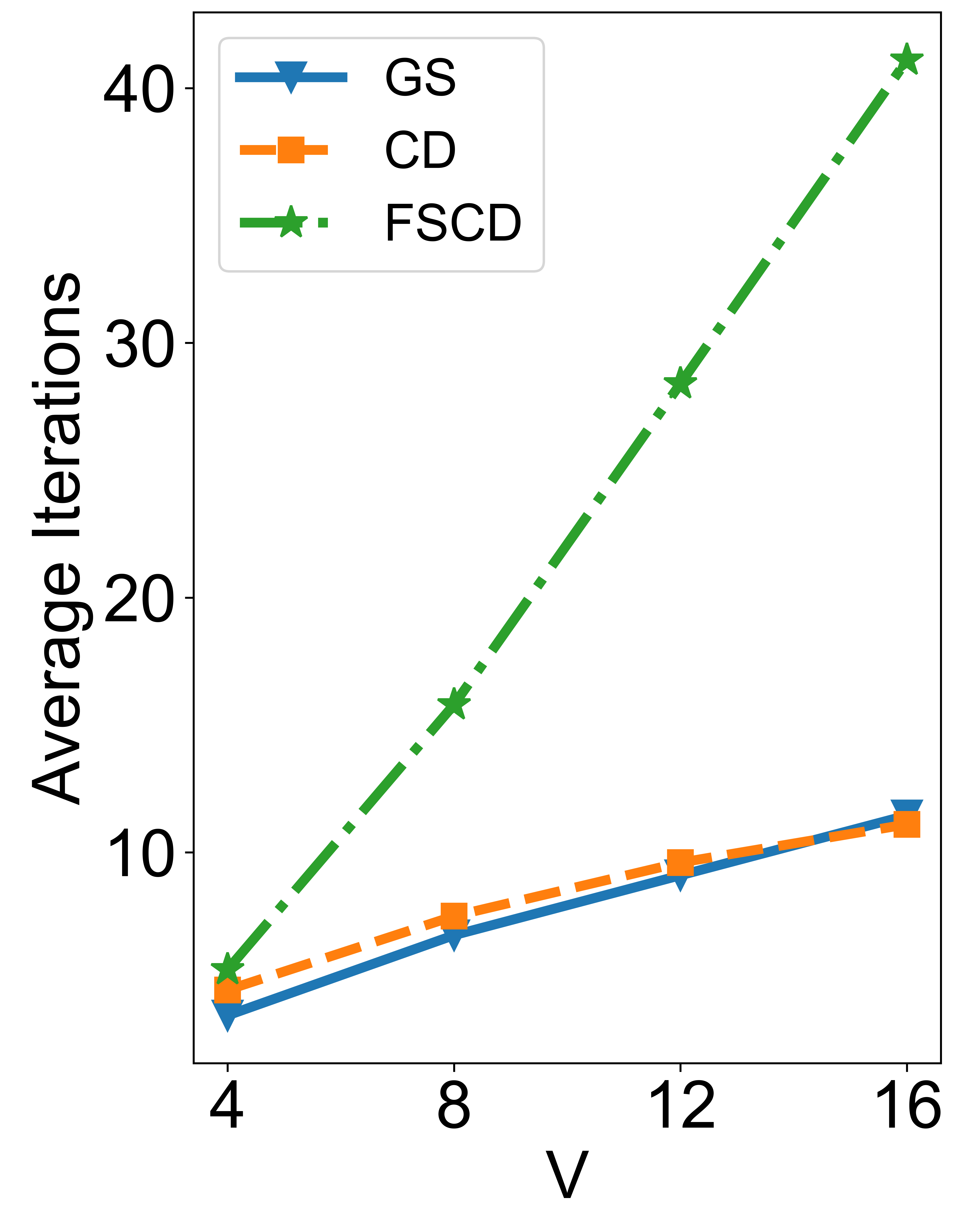} }	
	\subfigure[Average relative error.]{\label{fig19}	
		\includegraphics[width=0.22\textwidth]{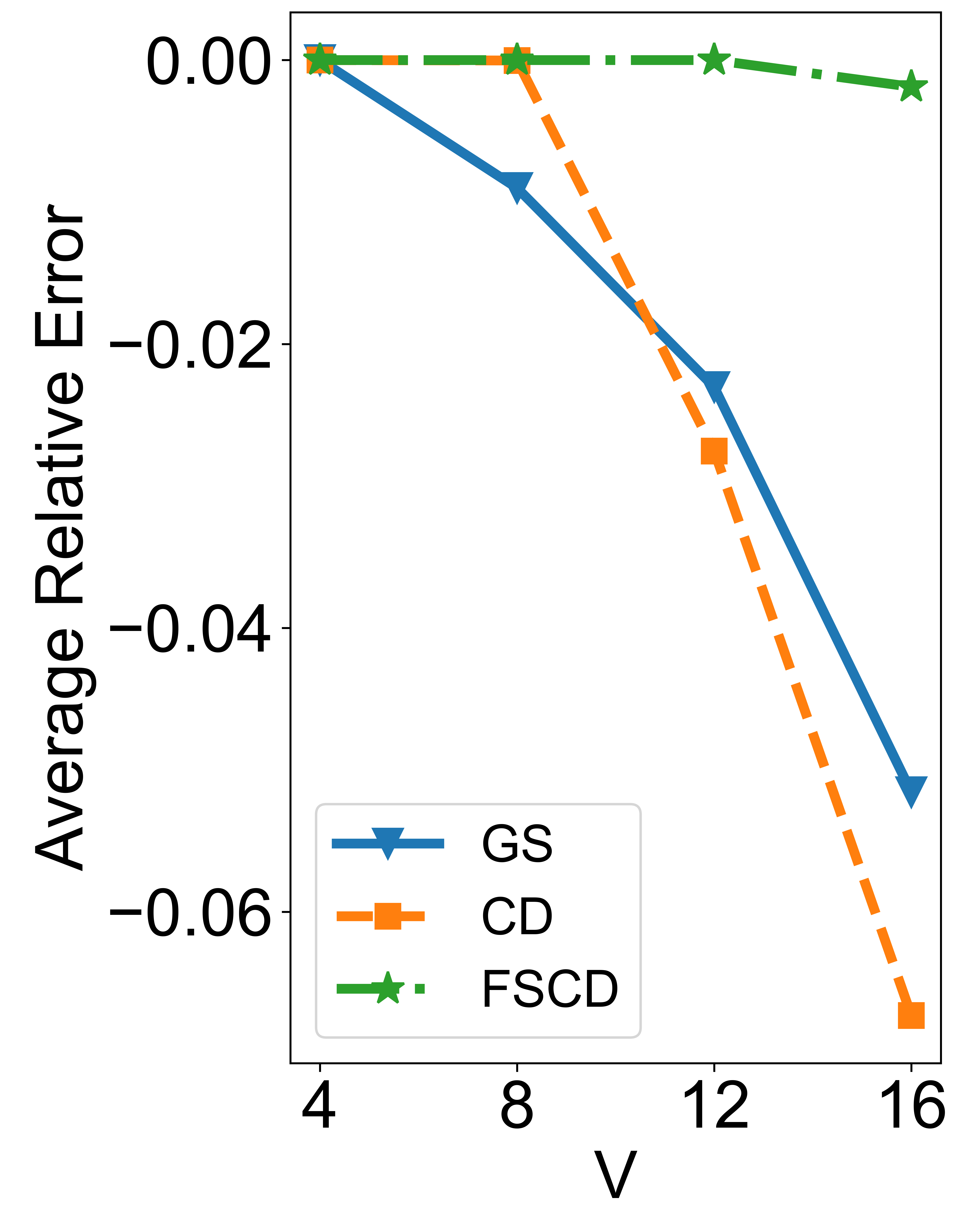} }
	\caption{Performance of proposed algorithms.}
	\label{fig18-19}
\end{figure}

(2) Dirichlet distribution\cite{hsu_measuring_2019}: for each device, the label follows Dirichlet distribution $\boldsymbol{p}_v\sim \text{\textbf{Dir}}(\alpha\boldsymbol{p})$, where $\alpha$ is the concentration parameter. A larger $\alpha$ means the distribution of devices is more similar to each other, and also to the global distribution $\boldsymbol{p}$. Each device is assumed to hold the same number of data samples. Due to the randomness of the choice $\boldsymbol{p}_v$, the total dataset of the Dirichlet distribution is naturally imbalanced.

For CIFAR-10, a convolutional neural network (CNN) is adopted with the following architecture: the first convolutional block consists of two $3\times3$ convolution layers with 32 channels, ReLU activation, followed by a $2\times2$ max pooling layer and 0.2 dropout; the second convolutional block contains two $3\times3$ convolution layers with 64 channels, ReLU activation, followed by a $2\times2$ max pooling layer and 0.3 dropout; a fully connected layer with 120 units and ReLU activation; and a final 10-unit softmax layer. The network takes an input of size $20\times3\times32\times32$ and produces an output of size $20\times10$. 

For CIFAR-100, ResNet-18\cite{he2016deep} is adopted. All the batch normalization layer in ResNet-18 is replaced by the group normalization layer to enable FL on heterogeneous data to converge \cite{wang2023batch}. The transmission deadlines for CIFAR-10 and CIFAR-100 are set to 2 s and 120 s, respectively. For both datasets, the learning rate is set to $\eta=0.1$ and the batch size is set to $b=32$. 

We compare our algorithm with four baselines:

(1) Best channel (BC)\cite{amiri_convergence_2021}: sort the devices by the channel gains in descending order, and do a best-effort schedule, i.e., schedule devices one by one until the channel bandwidth is not enough for the next device.

(2) Best norm (BN2)\cite{amiri_convergence_2021}: sort the devices by the gradient norms in descending order, and do a best-effort schedule.

(3) Power-of-choice (POC)\cite{cho_towards_2022}: randomly sample $V'$ devices, sort them by the accumulated loss in descending order, and do a best-effort schedule among the sampled devices.

(4) Fed-CBS\cite{zhang_fed-cbs_2023}: iteratively calculate the combinatorial upper confidence bounds of devices w.r.t the QCID value as the scheduling probability, and schedule according to the probability until the channel bandwidth is not enough for the next device.

For our method, we consider three variations:

(1) FedCGD-FSCD-Gc: estimate $G_c$ by class (only feasible when $l=1$), and use FSCD algorithm.

(2) FedCGD-FSCD: estimate $G$ and use FSCD algorithm.

(3) FedCGD-GS: estimate $G$ and use greedy scheduling.

\subsection{Performance of scheduling algorithms}
We first evaluate the performance of the proposed algorithms on the CIFAR-10 dataset. The data division follows the Dirichlet distribution, and the UMi-Street Canyon channel is adopted.

Fig. \ref{fig18-19} compares the performance of the aforementioned algorithms for P1, where the CD algorithm serves as a baseline. Results show that average iteration times increase almost linearly with the number of devices. The greedy scheduling uses iterations fewer than the number of devices to achieve an acceptable relative error of up to 5.16\%. When $V=16$, its relative error is even less than the CD method. The FSCD method reports a fantastic performance with a relative error 0.19\%, while not using too many iterations.

\begin{figure}[!t]
	\centering
	\subfigure[$\tau=1$.]{\label{fig1}		
	\includegraphics[width=0.45\textwidth]{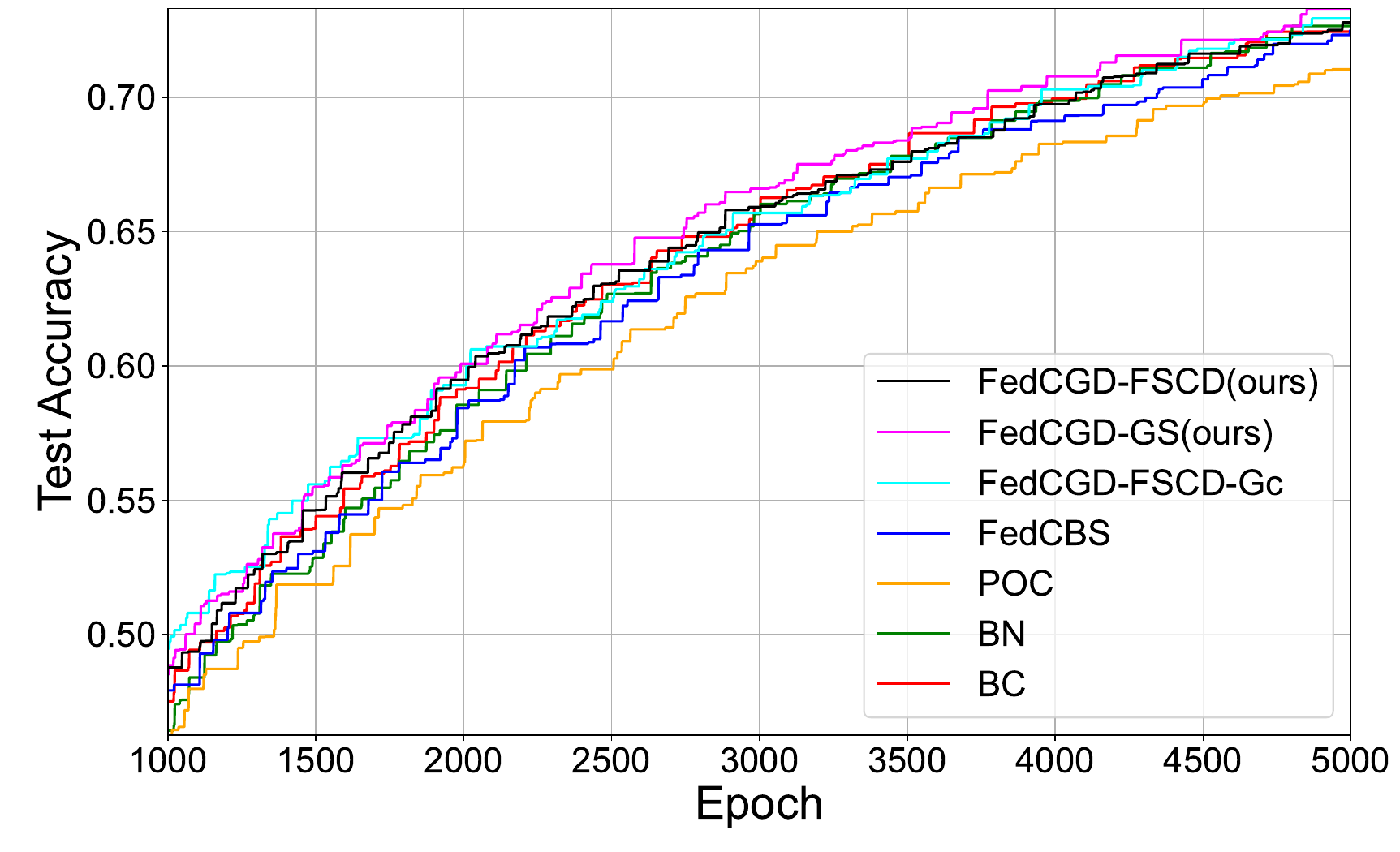} }	
	\subfigure[$\tau=5$.]{\label{fig2}	
	\includegraphics[width=0.45\textwidth]{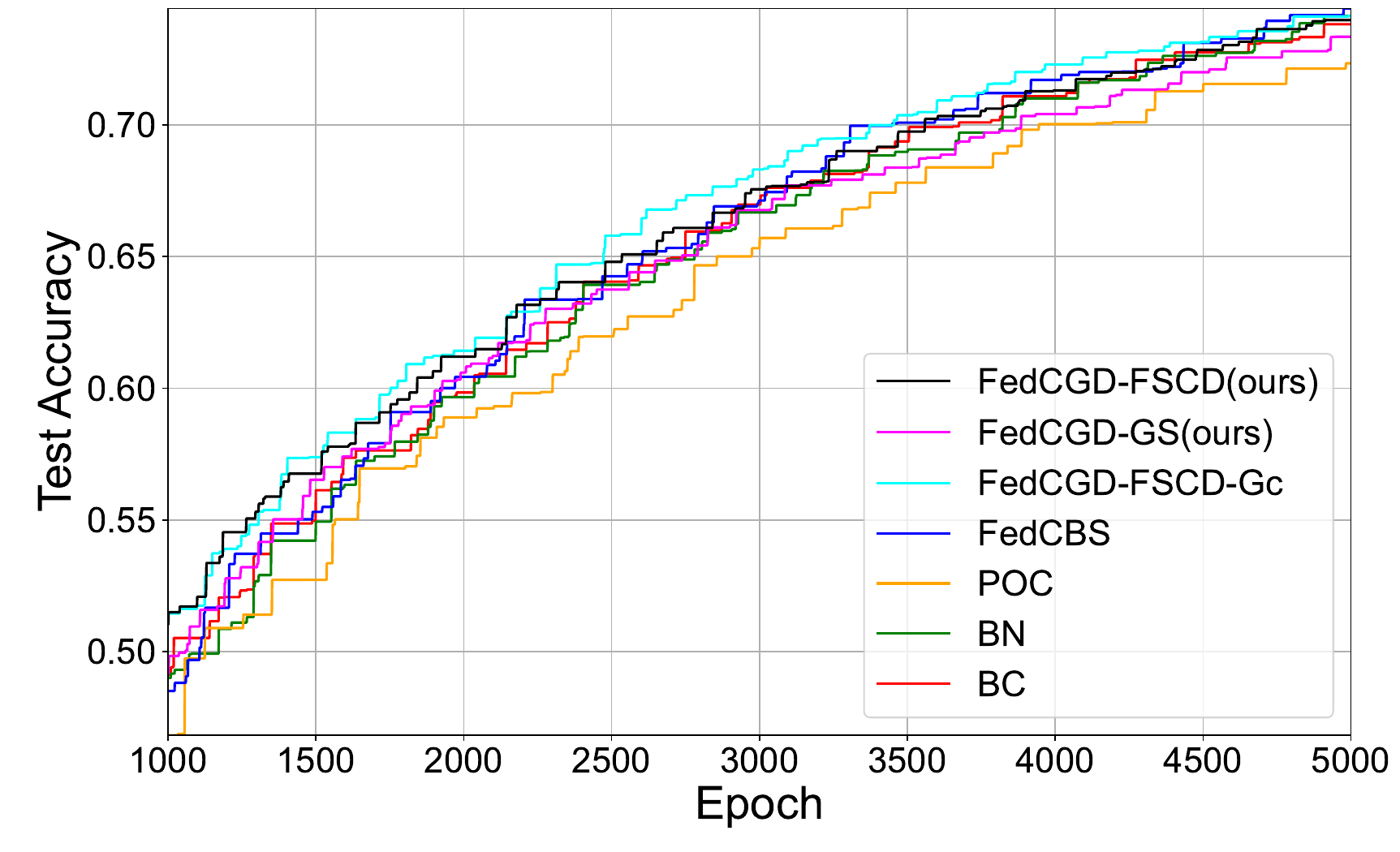} }
	\caption{Test accuracy of baselines with different local iterations.}
	\label{fig-2}
\end{figure}

\begin{figure}[t]
\centering
\includegraphics[width=0.45\textwidth]{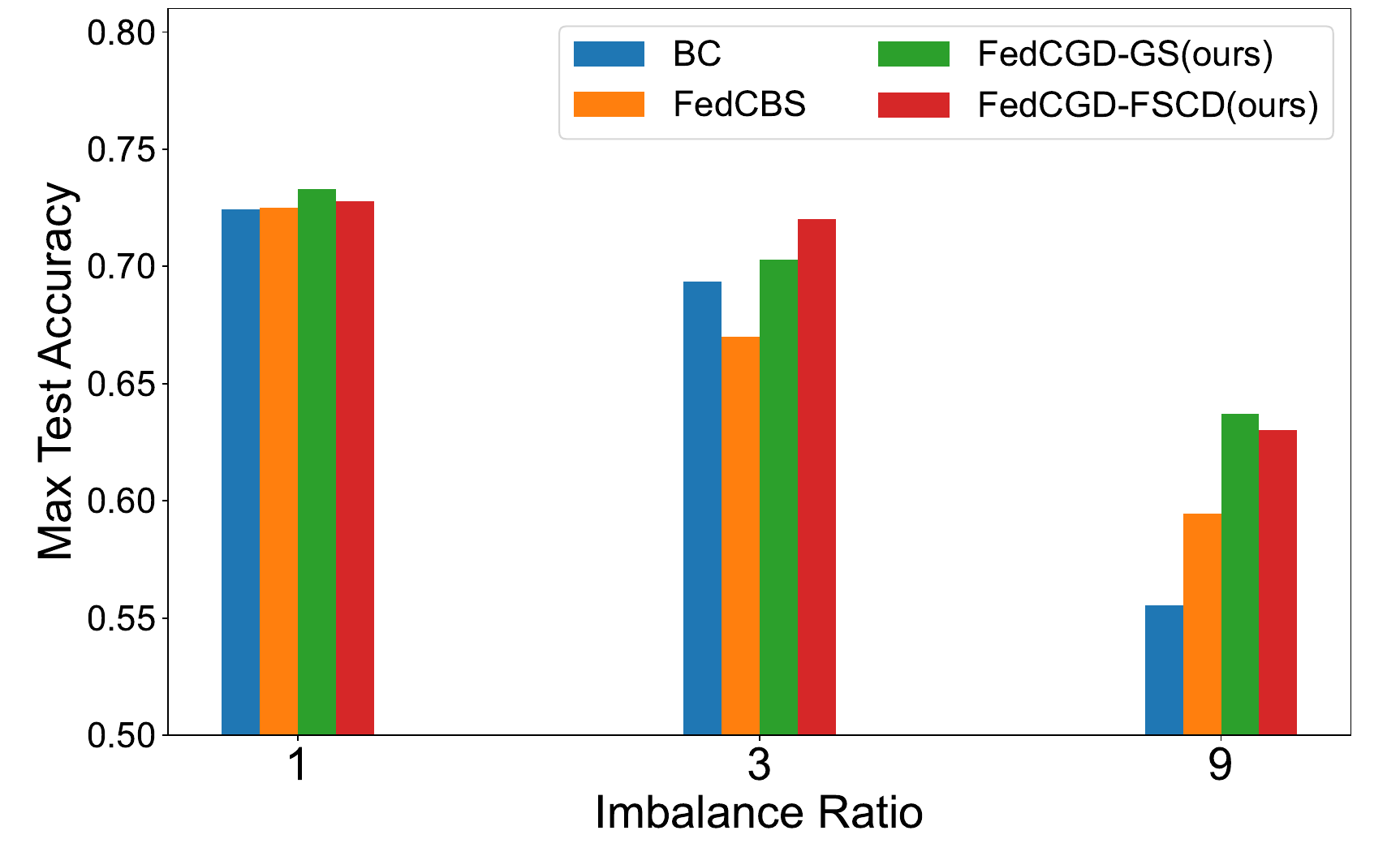}
\caption{Max test accuracy of baselines with different imbalance ratios.}
\label{fig24}
\end{figure}

\subsection{Performance of baselines on CIFAR-10}
\subsubsection{Balanced Total Dataset}

Firstly, Fig. \ref{fig-2} shows the performance of baselines and proposed algorithms listed above with an imbalance ratio $r=1$. The results demonstrate that our algorithms increase faster at the start of training, but finally, all the methods have similar performance. This is because when the total dataset is balanced, any scheduling method nearly equally chooses devices holding different classes of data throughout the training process, and thus the model learns the features of different classes sufficiently.  

\begin{table}[!t]
\centering
\caption{Performance of baselines}
\renewcommand{\arraystretch}{1.1} 
\begin{tabular}{|l|c|c|}
\hline 
Algorithm\rule[-1.5ex]{0pt}{3.0ex} & Avg. Scheduled Num. & Avg. WEMD \\ \hline
FedCGD-FSCD(ours) & 10.20 & \textbf{0.306}  \\ \hline
FedCGD-GS(ours) & \textbf{10.01} & 0.311 \\ \hline
FedCGD-FSCD-Gc & 10.10 & 0.310 \\ \hline
FCBS & 17.34 & 0.536 \\ \hline
POC & 13.71 & 0.621 \\ \hline
BN & 17.39 & 0.537 \\ \hline
BC & 17.67 & 0.527 \\ \hline
\end{tabular}\label{WEMD}
\end{table}

Specifically, FedCGD-GS reports the best accuracy when $\tau=1$, while FedCGD-FSCD-Gc and FedCGD-FSCD are just behind it. When $\tau=5$, FCBS and FedCGD-FSCD-Gc  achieve the highest accuracy, while FedCGD-FSCD and FedCGD-GS fall in the middle of the baselines. This is because when $\tau$ is larger than 1, the model iteration error appears and gradually dominates the CE difference. Although we assume a unified gradient bound $g$, devices may actually have different gradient bounds. Therefore, the difference in iteration error may overlap with the difference in device-level CGD and sampling variance, and thus reduce the advantage of our method.

Furthermore, Tab. \ref{WEMD} shows the number of scheduled devices and the corresponding WEMD. Comparing Tab. \ref{WEMD} with Fig. \ref{fig-2}, the accuracy of other baselines is positively correlated to the scheduled number, indicating that the accuracy gain of the better baselines, such as BC and FCBS, mainly comes from scheduling more devices; only our method FCGD schedules much fewer devices compared to baselines, while retaining relatively better performance. Specifically, FedCGD-FSCD-Gc achieves similar performance as FCBS, but schedules 41.8\% fewer devices. Tab. \ref{WEMD} also demonstrates that FCGD effectively reduces the WEMD of scheduled devices and thus reduces the device-level CGD.

Besides, comparing FedCGD-FSCD-Gc and FedCGD-FSCD, a more precise estimation for $G_c$ makes the accuracy increase faster at the start of training, but does not obviously increase the final accuracy. Therefore, we only adopt the algorithms that estimate $G$ in the subsequent simulations.

\subsubsection{Imbalanced Total Dataset}
Then we compare the performance of the baselines with different $r$. As is depicted in Fig. \ref{fig24}, when $r=3,9$, FedCGD-FSCD outperforms all baselines by a large gap. Specifically, when $r=9$, the maximum accuracy of FedCGD-FSCD is 4.3\% higher than the best baseline FCBS. 
The results demonstrate that when the total dataset is imbalanced, FCGD's ability to choose the important data, i.e., the minor classes, becomes more essential, while for other baselines, the minor classes are not well-trained, and the performance suffers.

\begin{figure}[!t]
	\centering
	\subfigure[$\alpha=0.1$.]{\label{fig11}			
		\includegraphics[width=0.13\textwidth]{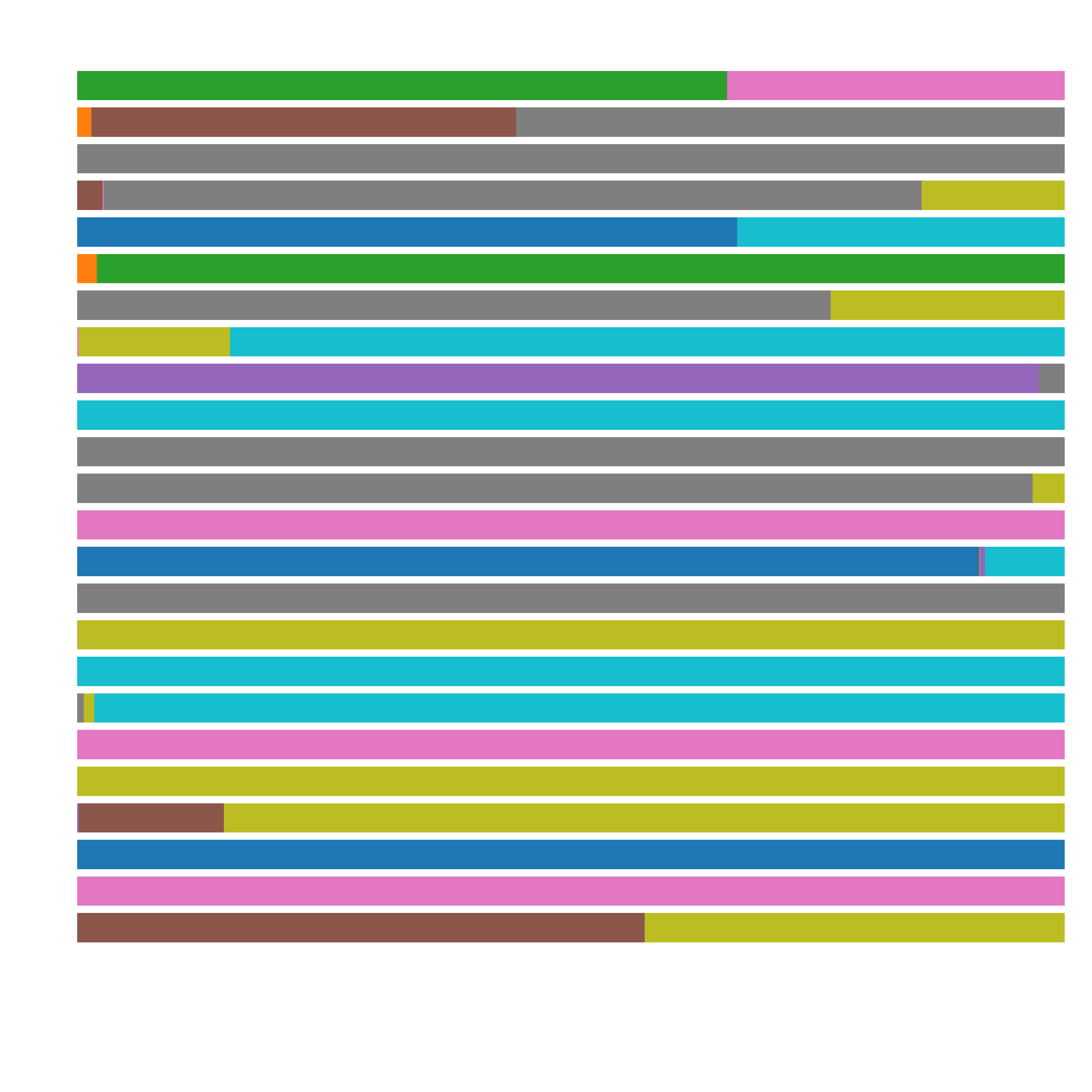} }		
	\subfigure[$\alpha=1$.]{\label{fig12}	
		\includegraphics[width=0.13\textwidth]{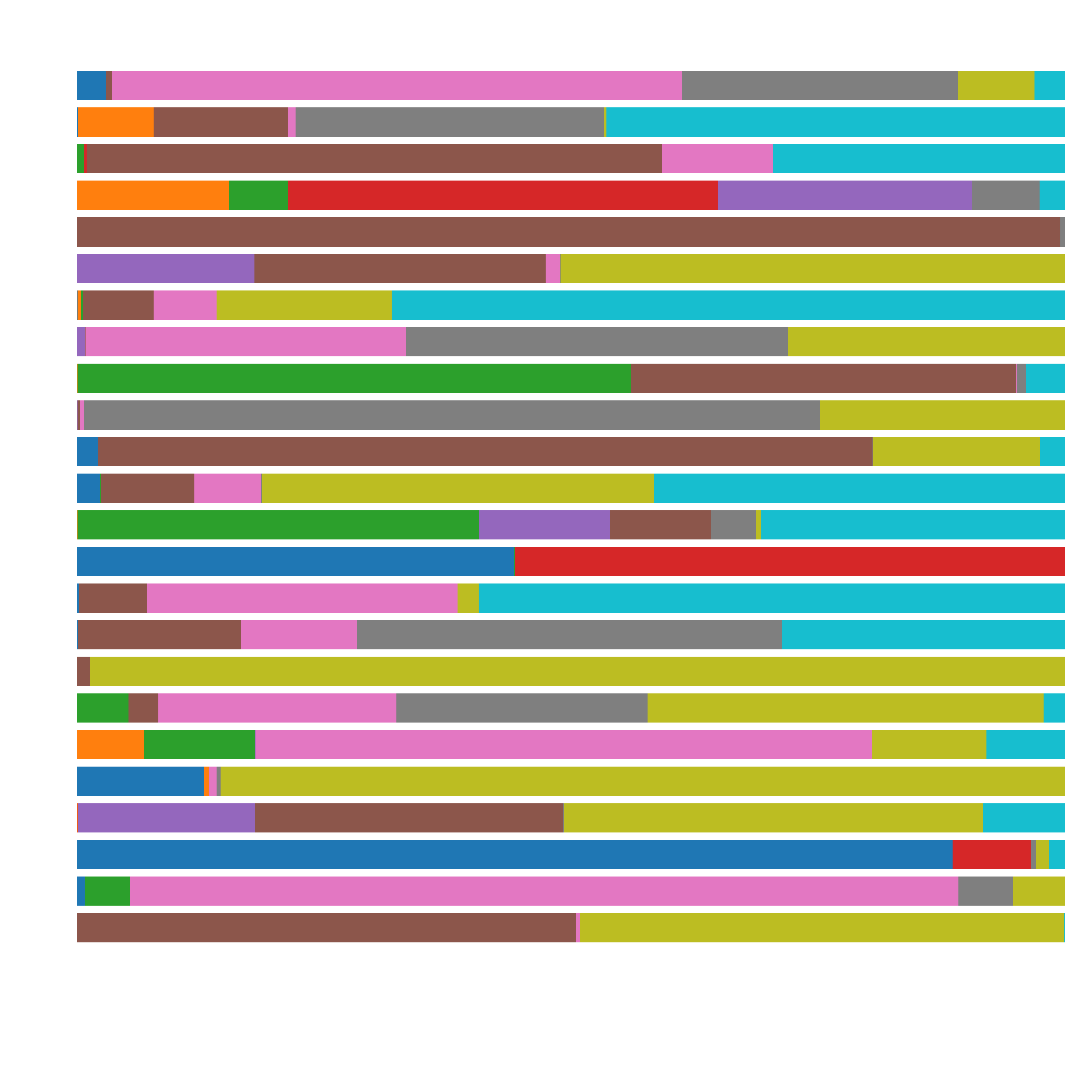} }
    \subfigure[$\alpha=10$.]{\label{fig13}	
		\includegraphics[width=0.13\textwidth]{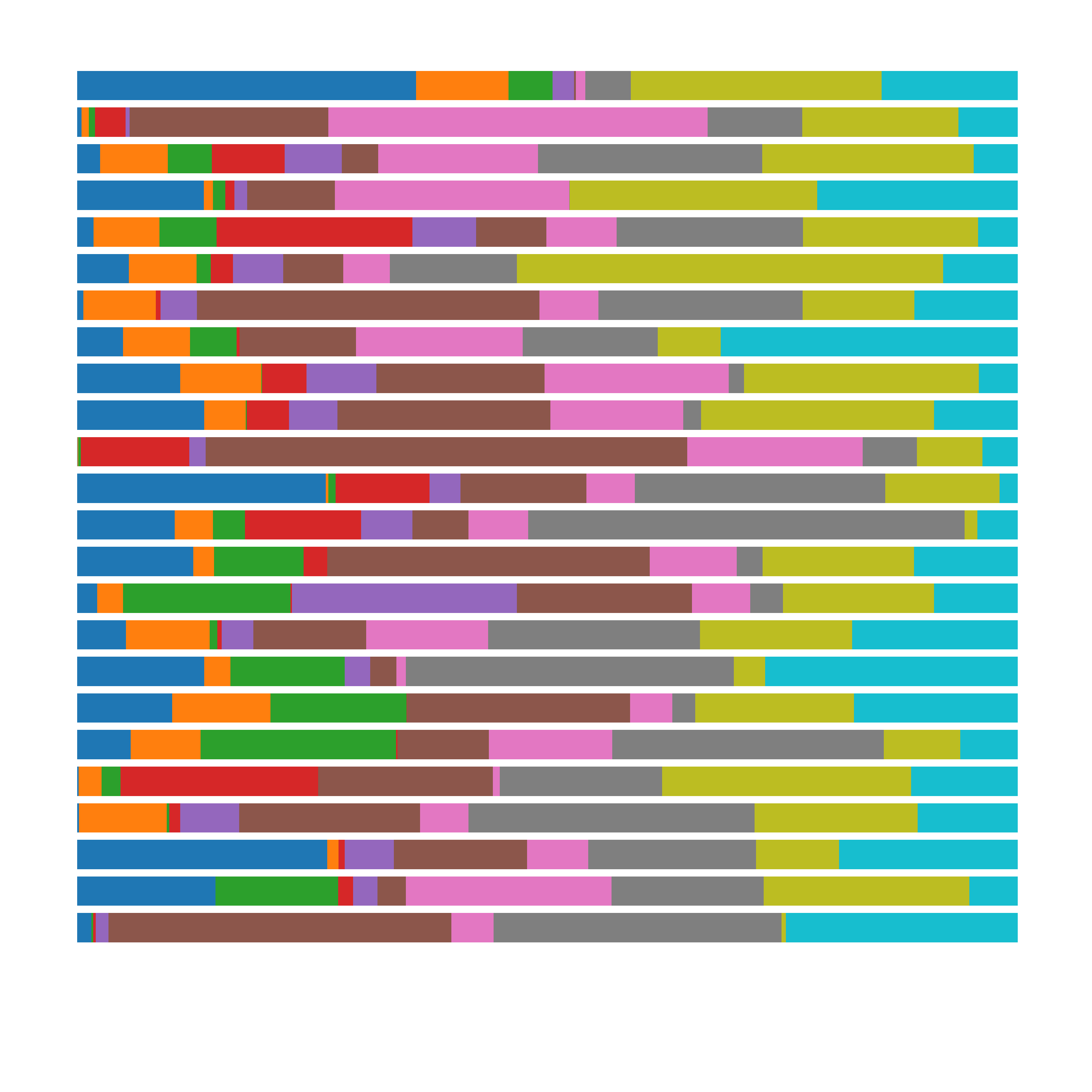} }
	\caption{Dirichlet data distribution with different $\alpha$. Each color represents one class of data.}
	\label{fig11-13}
\end{figure}

\begin{figure}[!t]
\centering
\includegraphics[width=0.49\textwidth]{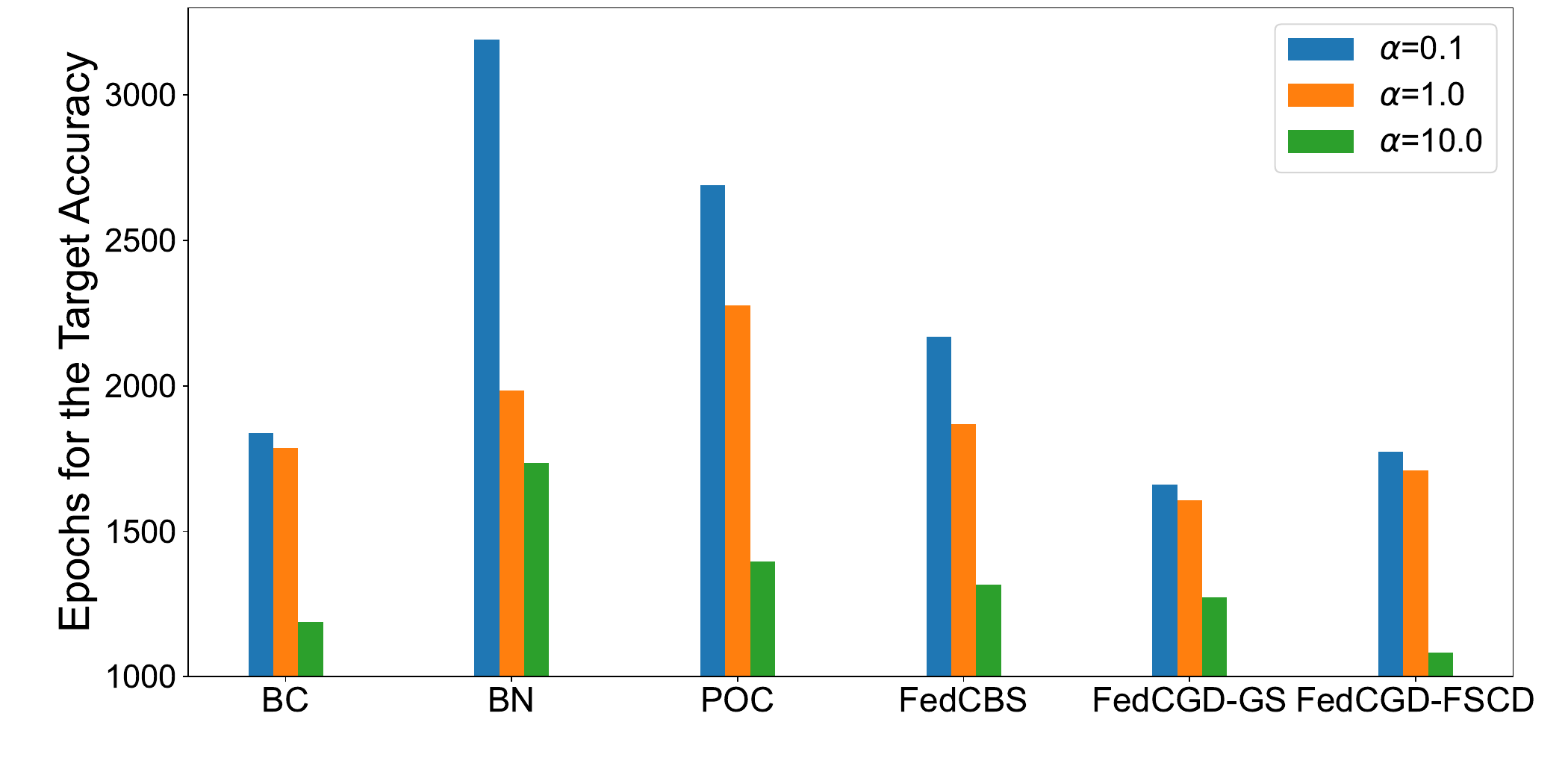}
\caption{Epochs to reach the target test accuracy of baselines with different $\alpha$.}
\label{fig15}
\end{figure}

\begin{figure}[!t]
\centering
\includegraphics[width=0.49\textwidth]{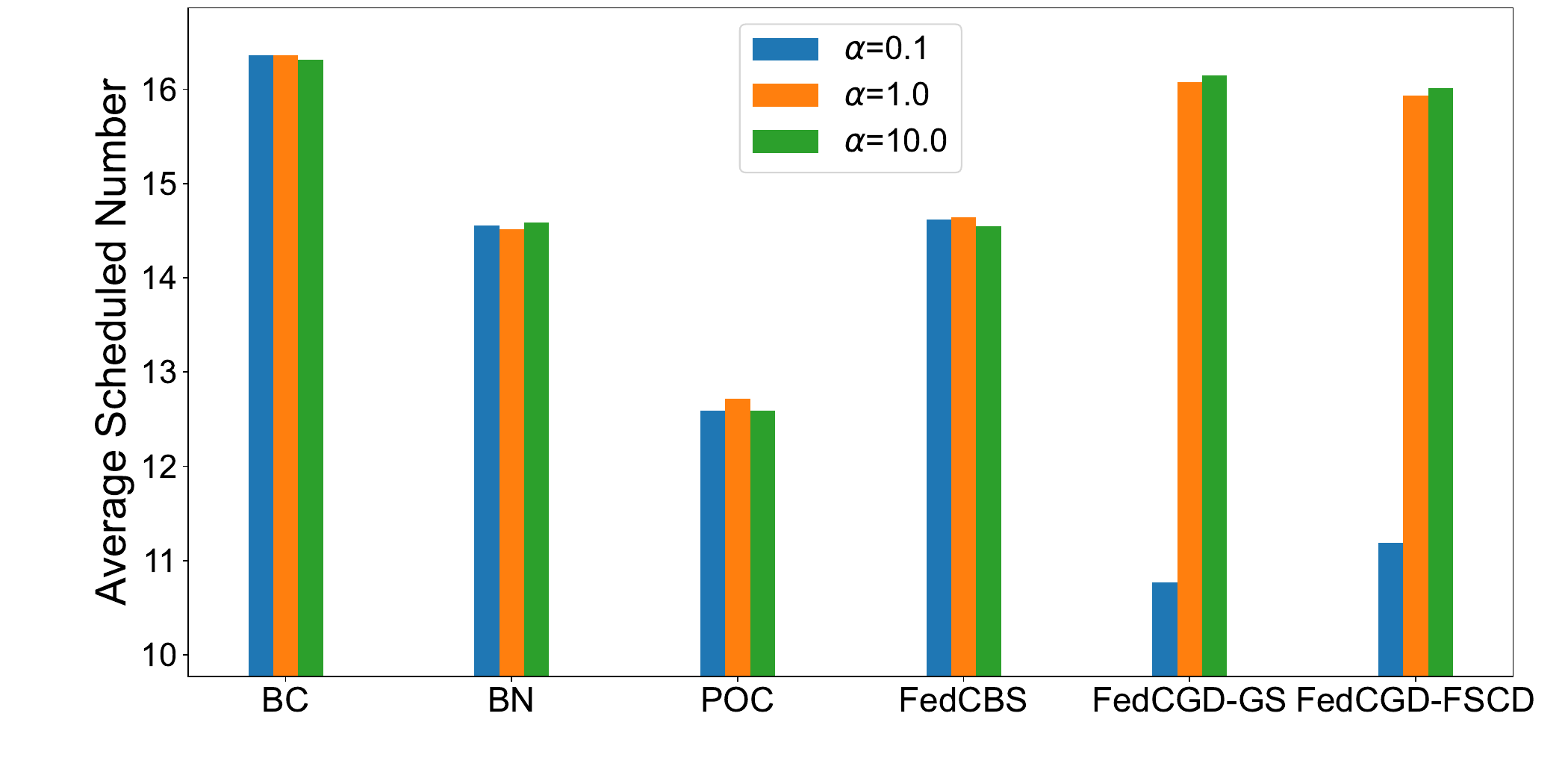}
\caption{Average scheduled number of baselines with different $\alpha$.}
\label{fig14}
\end{figure}

\subsubsection{Dirichlet Distribution}
Then we consider that the training datasets follow a Dirichlet distribution. $\alpha=0.1,1,10$ are considered
, where the data distribution of these cases is shown in Fig. \ref{fig11-13}.
Distribution of $\alpha=0.1$ is similar to the sort and partition case of $l=1$; In the distribution of $\alpha=10$, each device holds samples of nearly all the classes; $\alpha=1$ is at the middle of them. 

As is illustrated in Fig. \ref{fig15}, FedCGD-FSCD and FedCGD-GS achieve the target accuracy 65\% using the fewest epochs when $\alpha=0.1, 1$. When $\alpha=10$, FedCGD-GS uses the fewest epochs, while FedCGD-FSCD uses the third fewest epochs, but is still very close to BC. Furthermore, as is shown in Fig. \ref{fig14}, the scheduled number of FedCGD-FSCD and FedCGD-GS increases as $\alpha$ increases, while that of other baselines is basically the same. This is because when $\alpha$ is small, it's easier to find a schedule set that has a lower WEMD, e.g., when each device has data of one class, simply scheduling ten devices with different classes leads to WEMD$=0$. However, when $\alpha$ is large, each device has multiple classes of data, and it's hard to reach a small WEMD. Therefore, FCGD tends to schedule more devices to decrease sampling variance. Compared with baselines, FCGD shows the ability to dynamically adjust the focus between WEMD and sampling variance, and thus guarantees consistently good performance in different parameter settings.

\begin{figure}[!t]
	\centering
	\subfigure[Test accuracy.]{\label{fig20}		
	\includegraphics[width=0.45\textwidth]{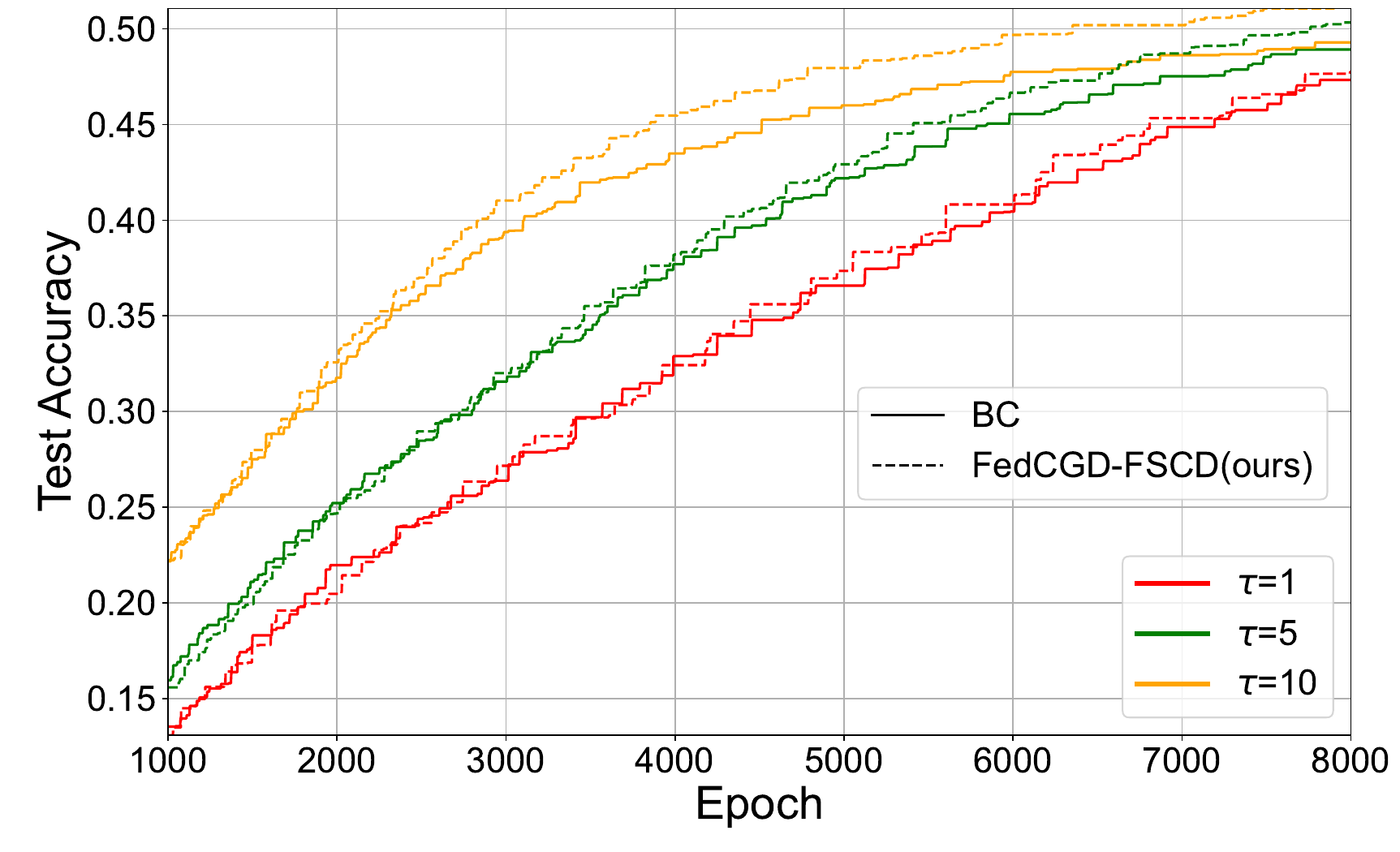} }	
	\subfigure[Estimation of $G$.]{\label{fig21}	
	\includegraphics[width=0.45\textwidth]{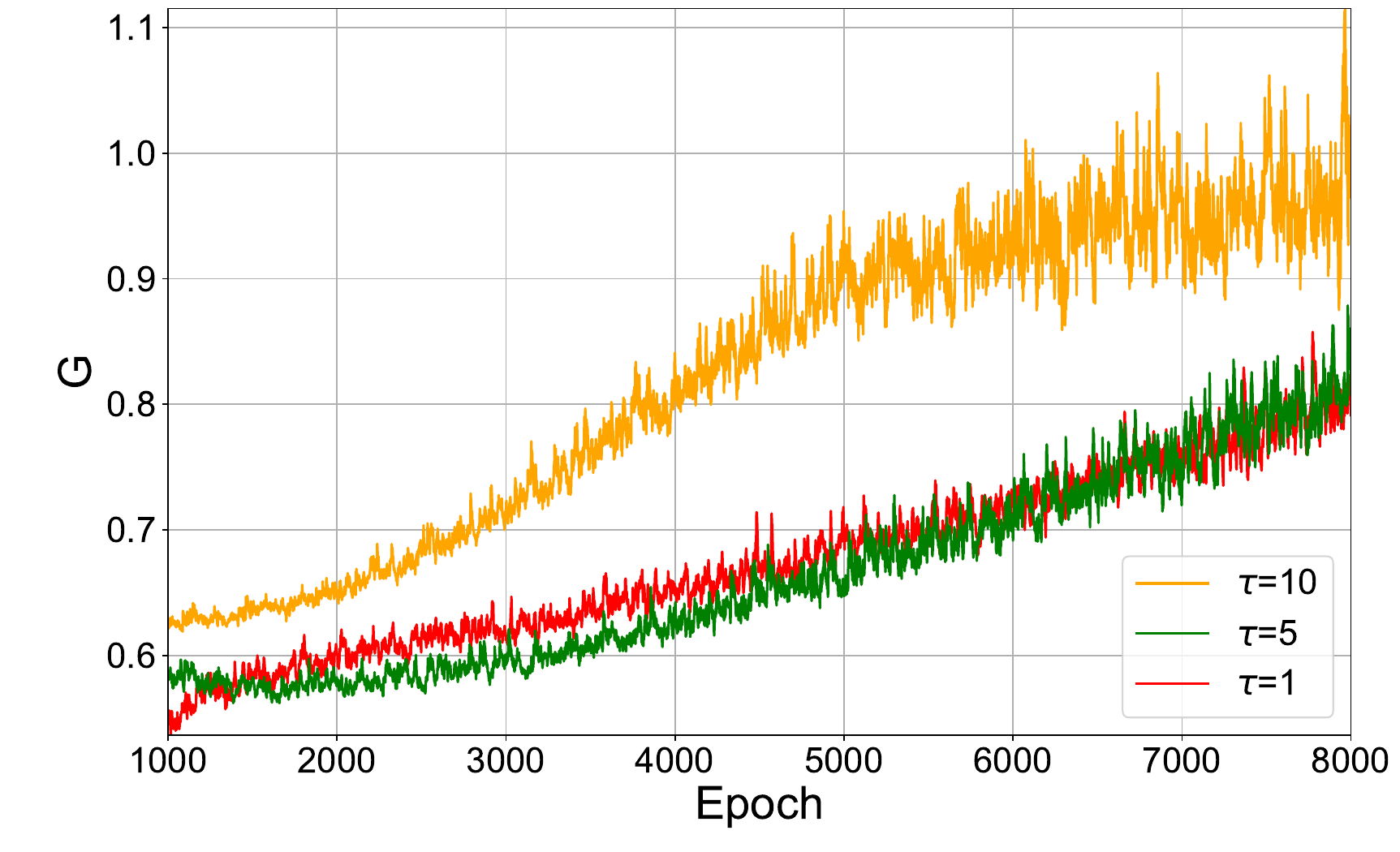} }
    \subfigure[Estimation of $\sigma$.]{\label{fig22}	
	\includegraphics[width=0.45\textwidth]{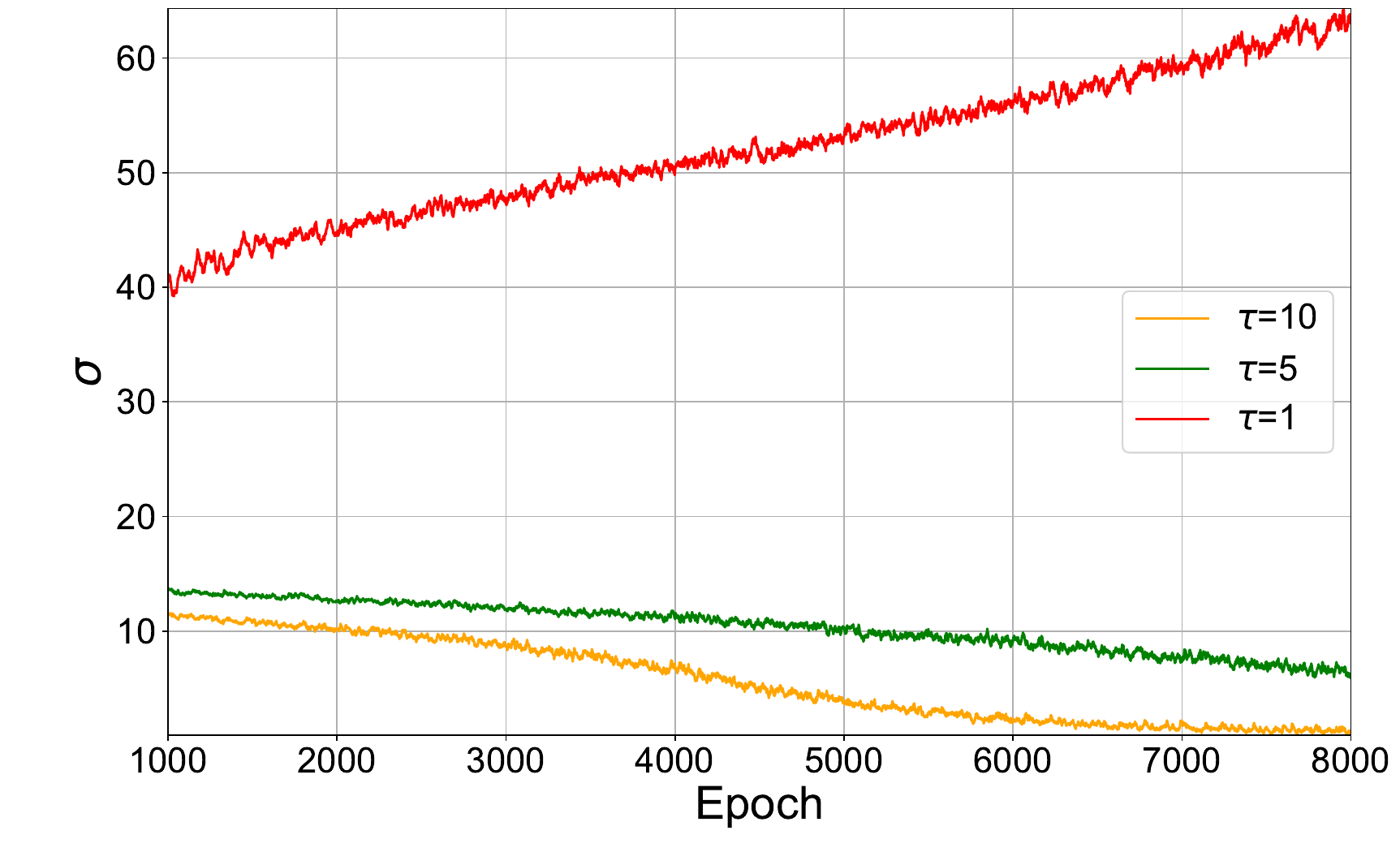} }
	\caption{Performance of baselines on the CIFAR-100.}
	\label{fig20-22}
    \end{figure}
    
\subsection{Performance of baselines on CIFAR-100}
For the CIFAR-100 dataset, we consider comparing the BC and FedCGD-FSCD methods. The data division adopts sort-and-partition with $l=5$.

As Fig. \ref{fig20} shows, the accuracy of the two methods is nearly the same when $\tau=1$, but when $\tau=5,10$, the accuracy of FedCGD-FSCD increases faster in the second half of the training. Specifically, when $\tau=10$, the final accuracy of FedCGD-FSCD is 2.5\% higher than BC. Fig. \ref{fig20-22} demonstrates a possible reason: the values of $G$ of different local updates all slowly increase during training, but the value of $\sigma$ has different evolving trends. When $\tau=1$, $\sigma$ also increases, which keeps sampling variance a high weight, so FCGD is similar to the BC method; when $\tau=5, 10$, $\sigma$ decreases, which makes device-level CGD more important, so FCGD shows an advantage. The results inspire us that $\frac{G}{\sigma}$ can serve as a leading indicator to remind us to consider device-level CGD: when $\frac{G}{\sigma}$ is small, the BC method is good enough; when $\frac{G}{\sigma}$ is large, the FCGD method is needed.
\section{conclusion}
\label{Sec-6}
In this article, we have investigated a collective measure for the data heterogeneity-aware scheduling in wireless FL. Convergence analysis has shown that the loss functions of both convex and non-convex FL are mainly impacted by the sum of multi-level CGDs, where sample-level CGD is further bounded by sampling variance. Device scheduling needs to strike a balance between them. 
Since device-level CGD is intractable in a realistic FL system, we have further transformed it into WEMD between the group and the global distribution by considering the classification problem. Afterward, an optimization problem has been formulated to minimize the sum of multi-level CGDs by balancing WEMD and sampling variance. The NP-hardness of the problem has been proved, and two heuristic algorithms have been proposed, where the greedy scheduling has $O(V^2)$ time complexity and the fix-sum coordinate descent achieves a relative error of 0.19\%. Finally, simulations have demonstrated that our methods achieve up to 4.2\% higher accuracy while scheduling up to 41.8\% fewer devices compared to baselines. 
Also, the estimated parameters $G$ and $\sigma$ have been proven to be an indicator of when to consider device-level CGD.
\onecolumn
\appendices
\section{Proof of Proposition \ref{prop1}}\label{app1}
To prove Proposition \ref{prop1}, we first define 
    \begin{align}
        \phi(j)=F\left(\boldsymbol{w}^{(j)}\right)-F\left(\boldsymbol{w}^*\right),
        \Tilde{\phi}(j)=F\left(\boldsymbol{v}^{(j)}\right)-F\left(\boldsymbol{w}^*\right).
    \end{align}
    As is proved in \cite{wang2019adaptive}, we have $\phi(j), \Tilde{\phi}(j)>0$ and $\phi(j)>\Tilde{\phi}\left(j+1\right)$. Besides, according to the definition, we obtain
    \begin{align}
        \phi(j)=\Tilde{\phi}(j),j\nmid j\label{phi},
        \phi(j)\ge\Tilde{\phi}(j),j\mid j.
    \end{align}
    Next, we prove the following lemma.
    \begin{lemma}\label{lemma5}
        For any $j$, when $\hat{\eta}\le\frac1{\beta}$ , we have
        \begin{equation}
            \frac1{\Tilde{\phi}\left(j+1\right)}-\frac1{\phi(j)}\ge\omega\hat{\eta}\left(1-\frac{\beta\hat{\eta}}{2}\right),
        \end{equation}
        where $\omega=\min\limits_j\frac1{\left\lVert{\boldsymbol{v}^{\left(j\right)}-\boldsymbol{w}^*}\right\rVert^2}$..
    \end{lemma}
\begin{proof}
        From Lemma 5 in \cite{wang2019adaptive}, we have
        \begin{equation}
            \Tilde{\phi}\left(j+1\right)-\phi(j)\le-\hat{\eta}\left(1-\frac{\beta\hat{\eta}}2\right)\left\lVert{\nabla F\left(\boldsymbol{v}^{(j)}\right)}\right\rVert^2.
        \end{equation}
        Equivalently,
        \begin{equation}\label{lemma1_1}
            \Tilde{\phi}\left(j+1\right)\le\phi(j)-\hat{\eta}\left(1-\frac{\beta\hat{\eta}}2\right)\left\lVert{\nabla F\left(\boldsymbol{v}^{(j)}\right)}\right\rVert^2.
        \end{equation}
        Furthermore, the convexity condition and Cauchy-Schwarz inequality give that
        \begin{align}
            \phi(j)&=F\left(\boldsymbol{v}^{(j)}\right)-F\left(\boldsymbol{w}^*\right)
            \le \nabla F\left(\boldsymbol{v}^{(j)}\right)^T\left(\boldsymbol{v}^{(j)}-\boldsymbol{w}^*\right)
            \le\left\lVert{\nabla F\left(\boldsymbol{v}^{(j)}\right)}\right\rVert \left\lVert{\boldsymbol{v}^{(j)}-\boldsymbol{w}^*}\right\rVert.\label{lemma1_2}
        \end{align}
        Take \eqref{lemma1_1} into \eqref{lemma1_2}, and we get
        \begin{equation}
            \Tilde{\phi}\left(j+1\right)\le\phi(j)-\hat{\eta}\left(1-\frac{\beta\hat{\eta}}2\right)\frac{\phi(j)^2}{\left\lVert{\boldsymbol{v}^{(j)}-\boldsymbol{w}^*}\right\rVert^2}.
        \end{equation}
        Dividing both sides by $\phi(j)\Tilde{\phi}\left(j+1\right)$, we obtain
        \begin{align}
            \frac1{\phi(j)} 
            &\le
            \frac1{\Tilde{\phi}\left(j+1\right)}-\hat{\eta}\left(1-\frac{\beta\hat{\eta}}2\right)\frac{\phi(j)}{\Tilde{\phi}\left(j+1\right)\left\lVert{\boldsymbol{v}^{(j)}-\boldsymbol{w}^*}\right\rVert^2}\\
            &\le 
            \frac1{\Tilde{\phi}\left(j+1\right)}-\hat{\eta}\left(1-\frac{\beta\hat{\eta}}2\right)\frac{1}{\left\lVert{\boldsymbol{v}^{(j)}-\boldsymbol{w}^*}\right\rVert^2}\\
            &\le 
            \frac1{\Tilde{\phi}\left(j+1\right)}-\omega\hat{\eta}\left(1-\frac{\beta\hat{\eta}}2\right)\label{lemma1_3},
        \end{align}
        where the second inequality comes from (\ref{phi}), and the third inequality holds because $\left\lVert{\boldsymbol{v}^{\left(j+1\right)}-\boldsymbol{w}^*}\right\rVert^2
        \le
        \left\lVert{\boldsymbol{v}^{(j)}-\boldsymbol{w}^*}\right\rVert^2$. Arranging \eqref{lemma1_3} we prove Lemma \ref{lemma5}.
\end{proof}

Now we can focus on proving Proposition \ref{prop1}.
  
        Using Lemma \ref{lemma5} and considering $j\in[0,J]$, we have
        
        \begin{align}
            \sum\limits_{j=0}^{J-1}
            \left[\frac1{\Tilde{\phi}\left(j+1\right)}-\frac1{\phi(j)}\right]
            =&\frac1{{\phi}\left(J\right)}-\frac1{\phi\left(0\right)}-\sum\limits_{j=1}^{J}\Big[\frac1{\phi\left(j\right)}-\frac1{\Tilde{\phi}\left(j\right)}\Big]\label{dcdiff}
            \ge J\omega\hat{\eta}\left(1-\frac{\beta\hat{\eta}}2\right).
        \end{align}
        Furthermore, each term in the sum of \eqref{dcdiff} can be expressed as
        \begin{align}
            \frac1{\phi\left(j\right)}-\frac1{\Tilde{\phi}\left(j\right)}
            &=
            \frac{\Tilde{\phi}\left(j\right)-\phi\left(j\right)}{\Tilde{\phi}\left(j\right)\phi\left(j\right)}\\
            &=
            \frac{\mathbb{E}\left[F\left(\boldsymbol{v}^{\left(j\right)}\right)-F\left(\boldsymbol{w}^{\left(j\right)}\right)\right]}{\Tilde{\phi}\left(j\right)\phi\left(j\right)}
            \ge 
            \frac{-\rho \mathbb{E}\left[U_k\right]}{\Tilde{\phi}\left(j\right)\phi\left(j\right)}\label{uphi},
        \end{align}
        where the inequality holds because of Assumption \ref{assu1}. Combining condition (3) in Proposition \ref{prop1} with \eqref{phi}, the denominator in the right-hand side of \eqref{uphi} can be bounded by
        \begin{align}
            \Tilde{\phi}\left(j\right)\phi\left(j\right)\ge
            \Tilde{\phi}^2\left(j\right)\ge\epsilon^2.\label{phiep}
        \end{align}
        Substituting \eqref{phiep} into \eqref{dcdiff}, we obtain
        \begin{align}
            \frac1{\phi\left(J\right)}
            &\ge\frac1{\phi\left(J\right)}-\frac1{\phi\left(0\right)}\\
            &\ge
            J\omega\hat{\eta}\left(1-\frac{\beta\hat{\eta}}2\right)-\sum\limits_{j=1}^{J}\frac{\rho \mathbb{E}\left[U_k\right]}{\epsilon^2}>0.\label{finalphi}
        \end{align}
        where the first inequality holds because $\phi(j)>0$ for any $j$.
        We rearrange \eqref{finalphi} and substitute $\hat{\eta}$ by $\eta \tau$ to get
        \begin{align}
            \phi\left(J\right)
            &=F\left(\boldsymbol{w}^{\left(J\right)}\right)-F\left(\boldsymbol{w}^*\right)
            \le \frac1{J\omega\eta\tau\left(1-\frac{\beta\eta\tau}2\right)-\sum\limits_{j=1}^{J}\frac{\rho \mathbb{E}\left[U_j\right]}{\epsilon^2}}\\
        \end{align}

\section{Proof of Lemma \ref{lemma1}}\label{app2}
We have the following for the FC difference:
\begin{align}
&\mathbb{E}\left\lVert\boldsymbol{w}^{(j)}-\boldsymbol{v}^{(j)}\right\rVert  \\
=&\mathbb{E}\left\lVert\sum\limits_{v\in\Pi^{(j)}} \alpha_{v}^{(j)} \left(\boldsymbol{w}_v^{(j)}
-\boldsymbol{v}^{(j)}\right)\right\rVert\\
=&\mathbb{E}\left\lVert\sum\limits_{v\in\Pi^{(j)}} \alpha_{v}^{(j)} \left[\left(\boldsymbol{w}_v^{(j)}
-\boldsymbol{\hat{w}}_v^{(j,\tau)}\right)
+\left(\boldsymbol{\hat{w}}_v^{(j,\tau)}-\boldsymbol{v}^{(j)}\right)
\right]\right\rVert\\
\leq &
\sum\limits_{v\in\Pi^{(j)}} \alpha_{v}^{(j)} \mathbb{E}\left\lVert\boldsymbol{w}_v^{(j)}
-\boldsymbol{\hat{w}}_v^{(j,\tau)}\right\rVert +
\mathbb{E}\left\lVert\sum\limits_{v\in\Pi^{(j)}} \alpha_{v}^{(j)} \left[
\boldsymbol{w}^{(j-1)}-\eta\tau\nabla f_{v,j}\left(\boldsymbol{w}^{(j-1)}\right)
-\left(\boldsymbol{w}^{(j-1)}-
\eta\tau\nabla F(\boldsymbol{w}^{(j-1)})\right)
\right]\right\rVert
\\
= &
\sum\limits_{v\in\Pi^{(j)}} \alpha_{v}^{(j)} \mathbb{E}\left\lVert\boldsymbol{w}_v^{(j)}
-\boldsymbol{\hat{w}}_v^{(j,\tau)}\right\rVert+
\eta\tau\underbrace{\mathbb{E}\left\lVert\sum\limits_{v\in\Pi^{(j)}} \alpha_{v}^{(j)} \left(\nabla f_{v,j}(\boldsymbol{w}^{(j-1)})-\nabla F(\boldsymbol{w}^{(j-1)})\right)
\right\rVert}_{A},\label{cediff}\\
\end{align}
where the inequality holds because of the absolute value inequality. The term (A) can be further decomposed by
\begin{align}
    A\le
    \mathbb{E}\left\lVert\sum\limits_{v\in\Pi^{(j)}} \alpha_{v}^{(j)} \left(\nabla f_{v,j}(\boldsymbol{w}^{(j-1)})-\nabla f_{v}(\boldsymbol{w}^{(j-1)})\right)\right\rVert
    + \left\lVert\sum\limits_{v\in\Pi^{(j)}} \alpha_{v}^{(j)} \nabla f_{v}(\boldsymbol{w}^{(j-1)})-\nabla F(\boldsymbol{w}^{(j-1)})\right\rVert\\
    \label{decomposeA}
\end{align}
Taking \eqref{decomposeA} into \eqref{cediff}, Lemma \ref{lemma1} is proved.

\section{Proof of Lemma \ref{lemma2}}\label{app3}
We have the following for sample-level gradient divergence:
\begin{align}
    &\mathbb{E}\left\lVert\sum\limits_{v\in\Pi^{(j)}} \!\alpha_{v}^{(j)} \left(\nabla f_{v,j}(\boldsymbol{w}^{(j-1)})\!-\!\nabla f_{v}(\boldsymbol{w}^{(j-1)})\right)\right\rVert\\
    &\le\sqrt{\mathbb{E}\left\lVert\sum\limits_{v\in\Pi^{(j)}} \!\alpha_{v}^{(j)} \left(\nabla f_{v,j}(\boldsymbol{w}^{(j-1)})\!-\!\nabla f_{v}(\boldsymbol{w}^{(j-1)})\right)\right\rVert^2}\\
    &=\sqrt{\mathbb{E}\left\lVert\frac1{|\Pi^{(j)}|b}\sum\limits_{v\in\Pi^{(j)}}\sum\limits_{i=1}^b\left[\nabla f_v(\boldsymbol{w}^{(j-1)},\boldsymbol{x}_{v,i})-\nabla f_v(\boldsymbol{w}^{(j-1)})\right]\right\rVert^2}\\
    &\le\sqrt{\frac1{(|\Pi^{(j)}|b)^2}\sum\limits_{v\in\Pi^{(j)}}\sum\limits_{i=1}^b\mathbb{E}\left\lVert\nabla f_v(\boldsymbol{w}^{(j-1)},\boldsymbol{x}_{v,i})-\nabla f_v(\boldsymbol{w}^{(j-1)})\right\rVert^2}\\
    &\le\frac{\sigma}{\sqrt{|\Pi^{(j)}|b}},
\end{align}
where the first inequality holds because $\mathbb{E}X\le\sqrt{\mathbb{E}\left[X^2\right]}$, the second inequality holds because the independence of gradients among different devices and the unbiasness of sampling makes the cross term to be 0, and the last inequality holds because of assumption 5) of Assumption \ref{assu1}.
        
\section{Proof of Lemma \ref{lemma3}}\label{app4}
\begin{proof}
    We can obtain the iterative form of $\mathbb{E}\left\lVert \boldsymbol{w}_v^{(j,t)} - \boldsymbol{\hat{w}}_v^{(j,t)} \right\rVert$ as following:
    \begin{align}
            \mathbb{E}\left\lVert \boldsymbol{w}_v^{(j,t)} - \boldsymbol{\hat{w}}_v^{(j,t)} \right\rVert
            &=\mathbb{E}\left\lVert \left(\boldsymbol{w}_v^{(j,t-1)} - \eta\nabla f_{v,j,t-1}\left(\boldsymbol{w}_v^{(j,t-1)}\right)\right) - 
            \left(\boldsymbol{\hat{w}}_v^{(j,t-1)} - \eta\nabla f_{v,j,t-1}\left(\boldsymbol{w}^{(j)}\right)\right)\right\rVert\\
            &=\mathbb{E}\left\lVert \left(\boldsymbol{w}_v^{(j,t-1)} - \boldsymbol{\hat{w}}_v^{(j,t-1)}\right)-\eta\left(\nabla f_{v,j,t-1}\left(\boldsymbol{w}_v^{(j,t-1)}\right) - 
            \nabla f_{v,j,t-1}\left(\boldsymbol{w}^{(j)}\right)\right) \right\rVert\\
            &\le \mathbb{E}\left\lVert \boldsymbol{w}_v^{(j,t-1)} - \boldsymbol{\hat{w}}_v^{(j,t-1)}\right\rVert
            + \eta\beta \mathbb{E}\left\lVert \boldsymbol{w}_v^{(j,t-1)} - \boldsymbol{w}^{(j)}\right\rVert, \label{localupdate2}
    \end{align}
    and the right hand side of Eq. \eqref{localupdate2} can be further bounded by
    \begin{align}
    \mathbb{E}\left\lVert \boldsymbol{w}_v^{(j,t-1)} - \boldsymbol{w}^{(j)}\right\rVert
    =\mathbb{E}\left\lVert\sum\limits_{u=0}^{t-2}\nabla f_u(\boldsymbol{w}_v^{(j, u)})\right\rVert
    \le \sum\limits_{u=0}^{t-2} \left\lVert
    \mathbb{E}\nabla f_u(\boldsymbol{w}_v^{(j, u)})\right\rVert
    \le (t-1)g,
    \end{align}
    where the last inequality holds because $\mathbb{E}X\le\sqrt{\mathbb{E}\left[X^2\right]}$ and assumption 4) of Assumption \ref{assu1}. Then by iterating we get
    \begin{align} \mathbb{E}\left\lVert \boldsymbol{w}_v^{(j)} - \boldsymbol{\hat{w}}_v^{(j,\tau)} \right\rVert\le
    \eta\beta g\sum\limits_{t=0}^{\tau-1}t
    =\frac12\eta\beta g\tau(\tau-1).
    \end{align}
\end{proof}

\section{Proof of Theorem \ref{theo2}}\label{app5}
Because of the $\beta$-smoothness of $F(\boldsymbol{w})$, we have

\begin{align}
    \mathbb{E}\left[F\left(\boldsymbol{w}^{(j)}\right)\right]
    \le \mathbb{E}\left[F\left(\boldsymbol{w}^{(j-1)}\right)\right] + 
    \underbrace{\left\langle\nabla F\left(\boldsymbol{w}^{(j-1)}\right), \mathbb{E}\left[\boldsymbol{w}^{(j)}-\boldsymbol{w}^{(j-1)}\right]\right\rangle}_B
    +\underbrace{\frac{\beta}2 \mathbb{E}\left\lVert \boldsymbol{w}^{(j)}-\boldsymbol{w}^{(j-1)} \right\rVert^2}_C.
\end{align}
Denote $\boldsymbol{U}_j = \boldsymbol{w}^{(j)}-\boldsymbol{v}^{(j)}$, so we have $U_j=\left\lVert\boldsymbol{U}_j\right\rVert$. Since
\begin{align}
    &\boldsymbol{w}^{(j)}-\boldsymbol{w}^{(j-1)}=\boldsymbol{w}^{(j)}-\left(\boldsymbol{v}^{(j)}+\eta\tau \nabla F(\boldsymbol{w}^{(j-1)})\right)=\boldsymbol{U}_j-\eta\tau \nabla F(\boldsymbol{w}^{(j-1)}),
\end{align}
we decompose terms B and C by
\begin{align}
    B=&\mathbb{E}\left[\left\langle\nabla F\left(\boldsymbol{w}^{(j-1)}\right), \boldsymbol{U}_j\right\rangle\right]
    - \eta\tau \left\lVert\nabla F\left(\boldsymbol{w}^{(j-1)}\right)\right\rVert^2\\
    \le&\left(\frac{\eta\tau}2-\eta\tau\right)\left\lVert\nabla F\left(\boldsymbol{w}^{(j-1)}\right)\right\rVert^2+
    \frac{1}{2\eta\tau}\mathbb{E}\left\lVert\boldsymbol{U}_j\right\rVert^2,\\
C=&\frac{\beta}2 \mathbb{E}\left\lVert\boldsymbol{U}_j-\eta\tau\nabla F\left(\boldsymbol{w}^{(j-1)}\right)\right\rVert^2
\le \beta \left(\eta^2\tau^2\left\lVert\nabla F(\boldsymbol{w}^{(j-1)})\right\rVert^2
+\mathbb{E}\left[U_j^2\right]\right),
\end{align}
so we can bound the squared norm of gradients by
\begin{align}
    \left(\frac{\eta\tau}2-\beta\eta^2\tau^2\right)\left\lVert\nabla F\left(\boldsymbol{w}^{(j-1)}\right)\right\rVert^2
    \le \mathbb{E}\left[F\left(\boldsymbol{w}^{(j-1)}\right)\right] - \mathbb{E}\left[F\left(\boldsymbol{w}^{(j)}\right)\right] + (\frac1{2\eta\tau}+\beta)\mathbb{E}\left[U_j^2\right].
\end{align}
Therefore, as long as $\eta\le\frac1{2\beta\tau}$, summing up over $j$ we have 
\begin{align}
    \frac1J\sum_{j=0}^{J-1}\left(\frac{\eta\tau}2-\beta\eta^2\tau^2\right)\left\lVert\nabla F\left(\boldsymbol{w}^{(j)}\right)\right\rVert^2
    \le \frac{\mathbb{E}\left[F\left(\boldsymbol{w}^{(0)}\right)\right] - \mathbb{E}\left[F\left(\boldsymbol{w}^{(J)}\right)\right]}J + (\frac1{2\eta\tau}+\beta)\frac1J\sum_{j=1}^{J}\mathbb{E}\left[U_j^2\right].
\end{align}
Then we focus on bounding $\mathbb{E}\left[U_j^2\right]$.

\begin{align}
\mathbb{E}\left[U_j^2\right]
= &\mathbb{E}\left\lVert\sum\limits_{v\in\Pi^{(j)}} \alpha_{v}^{(j)} \left[\left(\boldsymbol{w}_v^{(j)}
-\boldsymbol{\hat{w}}_v^{(j,\tau)}\right)
+\left(\boldsymbol{\hat{w}}_v^{(j,\tau)}-\boldsymbol{v}^{(j)}\right)
\right]\right\rVert^2
\\\le& 2\mathbb{E}\left\lVert\sum\limits_{v\in\Pi^{(j)}} \alpha_{v}^{(j)} \left(\boldsymbol{w}_v^{(j)}
-\boldsymbol{\hat{w}}_v^{(j,\tau)}\right)\right\rVert^2
+2\underbrace{\mathbb{E}\left\lVert\sum\limits_{v\in\Pi^{(j)}} \alpha_{v}^{(j)}\left(\boldsymbol{\hat{w}}_v^{(j,\tau)}-\boldsymbol{v}^{(j)}\right)
\right\rVert^2}_D,
\end{align}
where term D can be further decomposed by
\begin{align}
    D\le&\mathbb{E}\left[\left\lVert\sum\limits_{v\in\Pi^{(j)}} \alpha_{v}^{(j)} \left(\nabla f_{v,j}(\boldsymbol{w}^{(j-1)})-\nabla f_{v}(\boldsymbol{w}^{(j-1)})\right)\right\rVert
    + \left\lVert\sum\limits_{v\in\Pi^{(j)}} \alpha_{v}^{(j)} \nabla f_{v}(\boldsymbol{w}^{(j-1)})-\nabla F(\boldsymbol{w}^{(j-1)})\right\rVert\right]^2\\
    =&\mathbb{E}\left\lVert\sum\limits_{v\in\Pi^{(j)}} \alpha_{v}^{(j)} \left(\nabla f_{v,j}(\boldsymbol{w}^{(j-1)})-\nabla f_{v}(\boldsymbol{w}^{(j-1)})\right)\right\rVert^2 + \left\lVert\sum\limits_{v\in\Pi^{(j)}} \alpha_{v}^{(j)} \nabla f_{v}(\boldsymbol{w}^{(j-1)})-\nabla F(\boldsymbol{w}^{(j-1)})\right\rVert^2\\
    & + 2\mathbb{E}\left\lVert\sum\limits_{v\in\Pi^{(j)}} \alpha_{v}^{(j)} \left(\nabla f_{v,j}(\boldsymbol{w}^{(j-1)})-\nabla f_{v}(\boldsymbol{w}^{(j-1)})\right)\right\rVert
    \left\lVert\sum\limits_{v\in\Pi^{(j)}} \alpha_{v}^{(j)} \nabla f_{v}(\boldsymbol{w}^{(j-1)})-\nabla F(\boldsymbol{w}^{(j-1)})\right\rVert\\
    \le&\left(\frac{\sigma}{\sqrt{|\Pi^{(j)}|b}} + \Delta^{(j)}\right)^2,
\end{align}
where the last inequality holds because of Lemma \ref{lemma3} and the fact that $\mathbb{E}X\le\sqrt{\mathbb{E}\left[X^2\right]}$.

\section{Proof of Lemma \ref{lemma4}}\label{app6}

We set $C=1,|\mathcal{V}|=S,  p_{v,0}=r_v, p_0=\frac{c_{\text{sum}}}{2s},B_v=\frac Bs$ for P1, and the problem becomes

\begin{align}
\textbf{P2}: &\min_{\{x_{v}\}} \frac{\sigma}{\sqrt{\sum\limits_{v\in\mathcal{V}}x_{v}b}} + \left| 
\frac{ \sum\limits_{v \in \mathcal{V}}x_{v} r_v}{\sum\limits_{v \in \mathcal{V}} x_{v}}-\frac{c_{\text{sum}}}{2s} \right|G_c\\
&s.t. \quad x_{v} \in\{0,1\}, v\in\mathcal{V},\\
&\quad\quad \sum_{v\in\mathcal{V}} x_{v}\le s.\\
\end{align}

If we choose proper $\sigma,b,G_c$ such that $\frac{\sigma}{\sqrt{b}G_c}(\frac1{\sqrt{s-1}}-\frac1{\sqrt{s}})\ge \frac{c_{\text{sum}}}{2s}$ (this is easy to satisfy as long as $\frac{\sigma}{\sqrt{b}G_c}$ is large enough), then we have $\sum_{v\in\mathcal{V}} x^*_{v}= s$, because

\begin{align}
    \frac{\sigma}{\sqrt{sb}} + \left| 
\frac{ \sum\limits_{v \in \mathcal{V}}x_{v} r_v}{s}-\frac{c_{\text{sum}}}{2s} \right|G_c
\le G_c\left(\frac{\sigma}{\sqrt{sb}G_c} + \frac{c_{\text{sum}}}{2s}\right)\le\frac{\sigma}{\sqrt{(s-1)b}},
\end{align}
and thus we can not find a solution with $\sum_{v\in\mathcal{V}} x_{v}< s$ that has smaller objective value. Then the problem becomes

\begin{align}
\textbf{P3}: &\min_{\{x_{v}\}} \left| 
 \sum\limits_{v \in \mathcal{V}}x_{v} r_v-\frac{c_{\text{sum}}}{2} \right|\\
&s.t. \quad x_{v} \in\{0,1\}, v\in\mathcal{V},\\
&\quad\quad \sum_{v\in\mathcal{V}} x_{v}=s.\\
\end{align}

If the optimal value of P3 is 0, then the partition problem has a solution with set $\{i|x_i=1\}$ and its complementary set; If the optimal value of P3 is not 0, then the partition problem has no solution with a set of size $s$.

\twocolumn

\end{document}